\documentclass{article}

\usepackage[utf8]{inputenc} 
\usepackage[T1]{fontenc}    
\usepackage{hyperref}       
\usepackage{url}            
\usepackage{booktabs}       
\usepackage{amsfonts}       
\usepackage{nicefrac}       
\usepackage{microtype}      
\usepackage{amsmath,amssymb, amsthm}
\usepackage{tikz}
\usepackage[final, inline, nomargin]{fixme}
\usetikzlibrary{bayesnet}
\usepackage{subcaption}
\usepackage{caption}
\usepackage[accepted]{icml2021}

\renewcommand{\dot}[2]{\langle #1,#2\rangle}
\newcommand{\mindot}{\mathop{\mathrm{min.}{}}}

\newif\ifnotes\notesfalse

\ifnotes
\definecolor{mygrey}{gray}{0.50}
\newcommand{\notename}[2]{{\textcolor{red}{\footnotesize{\bf (#1:} {#2}{\bf ) }}}}

\else

\newcommand{\notename}[2]{{}}

\fi

\renewcommand{\P}{\mathbb{P}}

\newtheorem{theorem}{Theorem}
\newtheorem*{proposition*}{Proposition}
\newtheorem{lemma}{Lemma}

\newtheorem*{theorem*}{Theorem}

\begin{document}

%

%

\twocolumn[

\icmltitle{Graph cuts always find a global optimum for Potts models (with a catch)}

\begin{icmlauthorlist}
\icmlauthor{Hunter Lang}{mit}
\icmlauthor{David Sontag}{mit}
\icmlauthor{Aravindan Vijayaraghavan}{nw}
\end{icmlauthorlist}

\icmlaffiliation{mit}{MIT CSAIL, Cambridge MA, USA}
\icmlaffiliation{nw}{Northwestern University, Evanston IL, USA}

\icmlcorrespondingauthor{Hunter Lang}{hjl@mit.edu}

\icmlkeywords{Machine Learning, ICML}

\vskip 0.3in
]



\printAffiliationsAndNotice{}

\newcommand{\argmax}{\mathop{\mathrm{arg\,max}{}}}
\newcommand{\argmin}{\mathop{\mathrm{arg\,min}{}}}
\newcommand{\maximize}{\mathop{\mathrm{maximize}{}}}
\newcommand{\minimize}{\mathop{\mathrm{minimize}{}}}
\newcommand{\E}{\mathbb{E}}
\newcommand{\naive}{naive}
\renewcommand{\Pr}{\mathbb{P}}

\begin{abstract}
  We prove that the $\alpha$-expansion algorithm for MAP inference
  \emph{always} returns a globally optimal assignment for
  Markov Random Fields with Potts pairwise potentials, \emph{with a
    catch}: the returned assignment is only guaranteed to be optimal for an instance within a small perturbation of the original problem instance.
  In other words, all local minima with respect to expansion moves are
  global minima to slightly perturbed versions of the problem.
  On ``real-world'' instances, MAP assignments of small perturbations of the problem should be very similar to the MAP assignment(s) of the original problem instance.
  We design an algorithm that can certify whether this is the case in practice.
  On several MAP inference problem instances from computer vision, this algorithm certifies that MAP solutions to \emph{all} of these perturbations are very close to solutions of the
  original instance. These results taken together give a cohesive
  explanation for the good performance of ``graph cuts'' algorithms in
  practice. Every local expansion minimum is a global minimum in a small perturbation of the problem, and all of these global minima are close to the original solution. 
\end{abstract}

\section{Introduction}
\label{sec:intro}
Markov random fields are widely used for structured prediction in computer vision tasks such as image segmentation and stereopsis \citep{geman1984stochastic}, including in the modern ``deep'' era
\citep[e.g.,][]{zheng2015conditional}. Making predictions involves performing MAP inference. However, in general, exactly solving the MAP inference problem is NP-hard \citep{wainwright2008graphical} and one must resort to approximate inference.

``Graph cuts'' algorithms for approximate MAP inference in pairwise Markov random
fields have been very influential in computer vision. 
These algorithms are popular because they are simple and efficient,
and they return very high-quality
 solutions in practice \citep{szeliski2008comparative, kappes2015comparative}.
The \emph{$\alpha$-expansion} method of \citet{BoyVekZab01} starts with an arbitrary initial
labeling (an assignment of labels to variables), then iteratively makes ``expansion moves'' to improve the current labeling. At each step, the optimal expansion move of the current labeling can be computed very
efficiently by solving a binary minimum cut problem (hence the name
``graph cuts'').
The algorithm converges when no expansion moves can improve the labeling any further.

Algorithm~\ref{alg:expansion} summarizes the high-level algorithm steps. 
The $\alpha$-expansion algorithm is only guaranteed to return a \emph{local minimum} with respect to the moves made by the algorithm.
Figure \ref{fig:results} shows a globally optimal (MAP) labeling, which took over
four hours to obtain with an integer linear programming (ILP) solver,
and the local minimum returned by $\alpha$-expansion in less
than ten seconds. 
Although $\alpha$-expansion is only guaranteed to find a local minimum, the two assignments agree on over 99\% of the vertices.

Despite this good practical performance, the sharpest worst-case
theoretical guarantee for $\alpha$-expansion is that it obtains a
2-approximation to the objective value of the MAP labeling \citep{BoyVekZab01}. 
A 2-factor objective approximation often translates to a very weak guarantee for recovering the exact solution:
\citet{LanSonVij19} show that MAP inference problems from computer vision admit 2-approximate labelings that agree with the optimal assignment on fewer than 1\% of variables. Compare this to Figure \ref{fig:results}, where the expansion solution agrees with the exact solution on over 99\% of the nodes. 
Additionally, objective gap bounds obtained from primal-dual variants of $\alpha$-expansion are sometimes very close to one in practice \citep{komodakis2007fast}. 
Those bounds (which depend on the algorithm's initialization) show that graph-cuts algorithms
often vastly outperform their theoretical guarantees.
So a large gap exists between the worst-case guarantee of 2-approximation and the practical performance (in both objective value and recovery of the true solution) of the $\alpha$-expansion algorithm. 
\emph{Why are the local minima with respect to expansion moves so good in practice?}

\begin{algorithm}[t]
  \caption{$\alpha$-expansion algorithm}
  \label{alg:expansion}
  \begin{algorithmic}
  \STATE Initialize a labeling $x: V \to [k]$.
  
  \STATE {\tt improved} $\gets$ {\tt True}.
  
  \WHILE{\texttt{improved}}
  
  \STATE {\tt improved} $\gets$ {\tt False}
  
  \FOR{$\alpha \in [k]$}

    \STATE $x^{\alpha}\gets$ optimal $\alpha$-expansion of $x$.
    
    \IF{$obj(x^{\alpha}) < obj(x)$}
    \STATE $x\gets x^{\alpha}$.
    \STATE \texttt{improved} $\gets$ {\tt True}.
    \ENDIF
    \ENDFOR
    \ENDWHILE
  \STATE {\bfseries return} $x$.
\end{algorithmic}
\end{algorithm}
\begin{figure}[t]
\centering
\begin{subfigure}{.33\linewidth}
\centering
\includegraphics[width=\linewidth,height=60pt]{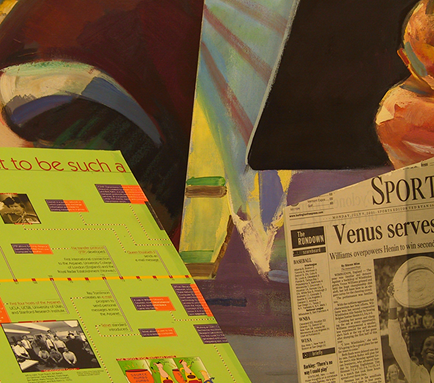}
\end{subfigure}%
\begin{subfigure}{.33\linewidth}
\centering
\includegraphics[width=\linewidth,height=60pt]{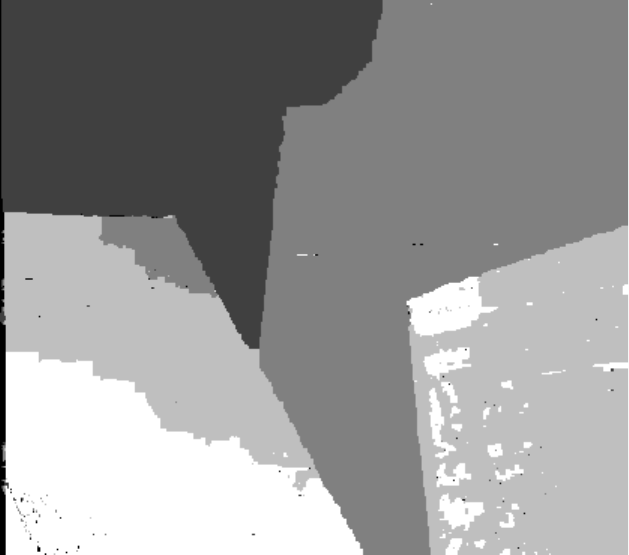}
\end{subfigure}%
\begin{subfigure}{.33\linewidth}
\centering
\includegraphics[width=\linewidth,height=60pt]{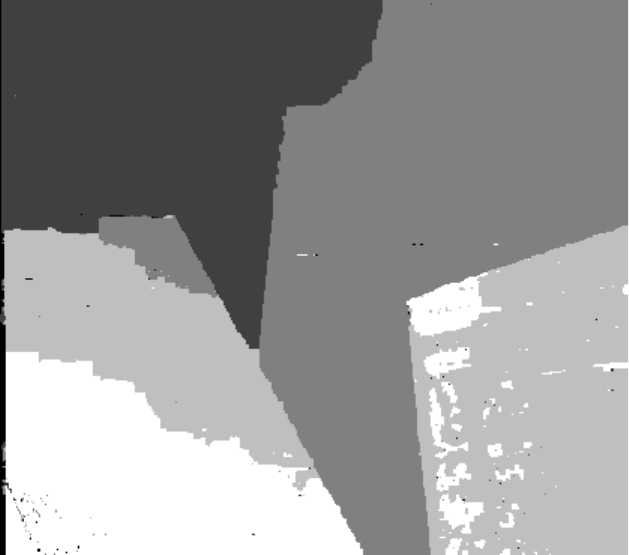}
\end{subfigure}
\caption{Left: image of {\tt venus} scene. Center: exact MAP depth labeling, found using ILP solver; Left: a local minimum w.r.t. expansion moves. The two labelings agree on over 99\% of nodes.}\label{fig:results}
\vspace{-1.45em}
\end{figure}

In this work, we prove a surprising structural result that
characterizes the local energy minima with respect to (w.r.t.) expansion moves.
Informally, we prove that for a widely used model (the \emph{Potts model}) all local expansion minima are actually \emph{global} minima of slightly
perturbed instances of the input problem. 
This result implies that the $\alpha$-expansion algorithm \emph{always} returns a MAP assignment---the ``catch'' is that the assignment is \textit{not} guaranteed to be optimal for the input instance, but rather for some closely-related instance.
In other words, we prove that when we run $\alpha$-expansion on an instance $I$, it always outputs a MAP solution to an instance $I'$ that is a small perturbation of $I$.

Our result implies that real-world instances should have no ``spurious'' local minima with respect to expansion moves.
This is because in practice, MAP solutions to all small perturbations of the problem instance should be very close to the MAP solution of the original instance.
We design an efficient algorithm to check whether this is truly the case. 
On real-world instances of MAP inference from computer vision, our algorithm certifies that all solutions to these perturbations are very close to the solution of the original instance. 
Our results thus give a theoretical explanation for the excellent empirical performance of $\alpha$-expansion and related graph cut methods like FastPD \citep{komodakis2007fast}. 
These algorithms naturally take advantage of the fact that solutions to all small perturbations tend to be close to the original solution in practice.

\section{Preliminaries}
\label{sec:prelim}

Before we discuss related work, we formally introduce the inference problem considered in this paper and fix notation.
Fix a constant $k$ and a graph $G=(V,E)$ with $|V|=n$, $|E| = m$. A \emph{labeling} of $G$ is a map $g: V\to [k]$. 
The (pairwise) MAP inference problem on $G$ can be written:
\begin{align*}
  \minimize_{g: V\to [k]}\sum_{u\in V}\theta_u(g(u)) + \sum_{(u,v) \in E}\theta_{uv}(g(u),g(v)).
\end{align*}
In this \emph{energy minimization} format, $\theta_u(i)$ is the
\emph{node cost} of assigning label $i \in \{1,\ldots, k\}$ to node
$u$, and $\theta_{uv}(i,j)$ is the \emph{edge cost} or \emph{pairwise energy} of simultaneously
assigning label $i$ to $u$ and $j$ to $v$. 
We assume without loss of generality that $\theta_u(i) \ge 0$ for all $(u,i)$.
Consider image segmentation: the nodes $u \in V$ correspond to image pixels,
and the edges $(u,v)\in E$ connect nearby pixels. The node costs
$\theta_u(i)$ can be set as the negative score of pixel $u$ for
segment $i$, and the pairwise terms $\theta_{uv}(i,j)$ can be set to
encourage nearby pixels to belong to the same
segment.

We can identify each labeling $g: V\to L$ with a point $x^g \in \{0,1\}^{nk + mk^2}$ defined by the following indicator functions:
\begin{align*}
x^g_u(i) &= \begin{cases}
  1 & g(u) = i\\
  0 & \mbox{otherwise}.
\end{cases}\\
x^g_{uv}(i,j) &= \begin{cases}
  1 & g(u) = i,\ g(v) = j\\
  0 & \mbox{otherwise}.
  \end{cases}
\end{align*}
The MAP inference problem can then be written as:
\begin{align*}
  \mindot_{g: V\to [k]}\sum_{u\in V}\sum_{i\in [k]}\theta_u(i)x^g_u(i) + \sum_{uv \in E}\sum_{i,j}\theta_{uv}(i,j)x^g_{uv}(i,j).
\end{align*}
The \emph{marginal polytope} $M(G)$ is defined as the convex hull of all $x^g$:
\[
M(G) \triangleq \mbox{conv}\left(\{x^g |\ g: V\to [k]\}\right).
\]
We can denote the coordinates of an arbitrary point in $x\in M(G)$ as $x_u(i),\ x_{uv}(i,j)$,
with $u \in V$, $uv \in E$, and $i$ and $j$ in $[k]$.
Collecting the objective coefficients $\theta_u(i)$ and $\theta_{uv}(i,j)$ in the vector $\theta = (\theta_u : u;\ \theta_{uv} : uv) \in \mathbb{R}^{nk+mk^2}$, we can rewrite MAP inference as a linear optimization over $M(G)$:
\[
\minimize_{x\in M(G)}\ \langle\theta, x\rangle,
\]
since the vertices of this polytope are precisely the points $x^g$ corresponding to labelings of $G$.
$M(G)$ typically lacks an efficient description, and optimizing a linear function over it is NP-hard in general \citep{wainwright2008graphical}.
The minimization problem above can be represented as the integer linear program (ILP):
\begin{alignat}{2}
  \label{eqn:ILP}
  \minimize_{x}\ &\dot{\theta}{x}\qquad\qquad\qquad&&\\
  \text{subj. to:}\ & \sum_{i}x_{uv}(i,j) = x_v(j)\ \ \ \ && \forall (u,v)\in E,\ j\in [k]\nonumber\\
                    & \sum_{j}x_{uv}(i,j) = x_u(i)\ \ \ \ && \forall (u,v)\in E,\ i\in [k]\nonumber\\
                    & \sum_{i}x_u(i) = 1 &&\forall u\in V\nonumber\\
                    & x_{uv}(i,j) \in \{0,1\}&&\forall (u,v),\ (i,j)\nonumber\\
                    & x_{u}(i) \in \{0,1\}&& \forall u,\ i.\nonumber
\end{alignat}
The feasible points of this ILP are precisely the vertices of $M(G)$.
A common approximate approach to MAP inference is to solve the following linear programming (LP) relaxation:
\begin{equation}
    \label{eqn:local-lp}
    \minimize_{x \in L(G)}\ \langle\theta, x\rangle,
\end{equation}
where $L(G)$ is the \emph{local polytope} defined by relaxing all the integrality constraints above from $\{0,1\}$ to $[0,1]$.
In general, $M(G)$ is a strict subset of $L(G)$.
The first two sets of constraints are called \emph{marginalization
  constraints}, and ensure the edge variables $x_{uv}(i,j)$ locally
match the ``marginals'' $x_u(i)$ and $x_v(j)$. The third set consists of \emph{normalization} constraints that ensure the $x_u$ variables sum to 1. Note that there are some redundant constraints in this formulation.
In the LP relaxation, the variables $x_u(i)$, $x_{uv}(i,j)$ correspond to potentially \emph{fractional} labelings of $G$, since we only have that $\sum_i x_u(i) = 1$.

We refer to \eqref{eqn:local-lp} as the \emph{local LP relaxation} of the MAP inference problem. 
This relaxation has been widely studied \citep{sontag2010approximate, wainwright2008graphical}, including in the context of stability and the ferromagnetic Potts model \citep{KleinbergTardos02, LanSonVij18, LanSonVij19, LanRedSonVij21}. 
\fxnote{put refs in better context} Many algorithms for approximate MAP inference can be related to this relaxation (e.g., MPLP \citep{globerson2008fixing} performs coordinate ascent in its dual), but we only use it here as a tool in our analysis. 
We say that \eqref{eqn:local-lp} is \emph{tight} on an instance of the MAP inference problem if there exists a vertex of $M(G)$ that is a solution to \eqref{eqn:local-lp}. This optimal vertex must correspond to an exact MAP labeling.

In this work, we focus on the \emph{ferromagnetic Potts
  model}, where the pairwise terms $\theta_{uv}(i,j) =
 w_{uv}\mathbb{I}[i\ne j]$, with $w_{uv} \ge 0$. That is, the cost of an edge only depends on whether the labels of its endpoints match. While seemingly simple, this model is popular in practice (for example, it accounts for several of the instances studied in \citet{kappes2015comparative}). We still use $\theta_{uv}(i,j)$ in what follows for notational convenience, but in the rest of our results, we assume $\theta_{uv}(i,j)$ takes this form. 
 MAP inference in this model is also called \emph{uniform metric labeling} \citep{KleinbergTardos02}, and it is NP-hard for variable $k \ge 3$, even when $G$ is planar
\citep{dahlhaus1992complexity}.

We often use the same symbol $x$ to refer to both a vertex of $M(G)$, referencing the values $x_u(i)$, $x_{uv}(i,j)$, and to a labeling of $G$, referencing $x(u)$.
This is justified because these two objects are in one-to-one correspondence.
For example, we write the objective value of a labeling $x: V\to [k]$ as $\dot{\theta}{x}$.
We also define the (normalized) \emph{Hamming distance} between two labelings $x$ and $x'$ as:
\[
\frac{1}{n}\sum_{u}\mathbb{I}[x(u) \ne x'(u)] = \frac{1}{2n}\sum_{u}\sum_i|x_u(i)-x'_u(i)|.
\]
Finally, for a fixed graph $G$ and a fixed $k$, we identify an instance of the MAP inference problem with its objective vector $\theta = (\theta_u : u, \theta_{uv} : uv)$.
\subsection{Expansion}
Let $x: V\to [k]$ be a labeling of $G$. For any label $\alpha \in [k]$, we say that $x'$ is an $\alpha$-expansion of $x$ if the following hold for all $u\in V$:
\begin{align*}
x(u) = \alpha &\implies x'(u) = \alpha,\\
x'(u) \ne \alpha &\implies x'(u) = x(u).
\end{align*}
That is, $x'$ may not shrink the region of nodes labeled $\alpha$---that region can only expand---and if $x'$ changes any label, the new label must be $\alpha$. The optimal $\alpha$-expansion move of $x$ can be found very efficiently by solving a minimum cut problem in an auxiliary graph $G^x(\alpha)$ \citep{BoyVekZab01}.
Algorithm \ref{alg:expansion} starts with an arbitrary labeling $x: V \to [k]$, then iteratively improves it by making expansion moves. The algorithm converges when there are no expansion moves that decrease the objective $\dot{\theta}{x}$. 
We say a labeling $x$ is a local minimum w.r.t. expansion moves if no expansion move of $x$ strictly decreases the objective.

The approximation guarantee for $\alpha$-expansion states that the objective of any local  minimum is at most the objective of the MAP solution $x^*$ plus the edge cost paid by $x^*$.
\begin{theorem}[\citep{BoyVekZab01} Theorem 6.1]
\label{thm:boykov}
Let $x$ be a local minimum w.r.t. expansion moves. Then
\[
\dot{\theta}{x} \le \dot{\theta}{x^*} + \sum_{uv}\theta_{uv}(x^*(u),x^*(v))
\]
In particular, $\dot{\theta}{x} \le 2\dot{\theta}{x^*}$.
\end{theorem}

\section{Related work}
\subsection{Perturbation stability}
\citet{LanSonVij18} define $(1,2)$-stable instances of uniform metric
labeling as those instances whose MAP solution does not change when any subset of edges $S \subset E$ can have the weights $w_{uv}$ multiplied by an edge-dependent $\gamma_{uv} \in [1,2]$. 
They prove that $\alpha$-expansion recovers the exact MAP
solution on $(1,2)$-stable instances. 
As is typically the case in work on perturbation stability, few guarantees are given for instances that do not satisfy the stability definition.
Unfortunately, the real-world instances that motivated \citet{LanSonVij18}'s work are not stable---the requirement of stability that the solution doesn't change \emph{at all} turns out to be too strict to be practical \citep{LanSonVij19}. Our results are much more general, since they apply to \emph{any} instance (stable or not). 
To go beyond stability, \citet{LanSonVij19} showed that an LP relaxation has approximate recovery guarantees when ``blocks'' (sub-instances) of the instance are perturbation stable, and \citet{LanRedSonVij21} showed that perturbation-stable instances are still approximately solvable after being corrupted by noise. 
Neither of these works gives a guarantee for graph cuts.

\subsubsection{Checking stability}
\citet{LanSonVij19} designed algorithms for checking stability and
sub-instance stability for uniform metric labeling that are based on
solving (a series of) integer linear programs. Surprisingly, we show
that our algorithm, which bounds the performance of all possible $\alpha$-expansion minima,
is \emph{computationally efficient} once an exact MAP solution $x^*$ is known. 

\subsection{Primal-dual graph cut algorithms}
\citet{komodakis2007fast} showed how to interpret $\alpha$-expansion as a primal-dual algorithm for solving the energy minimization problem. This view enables the algorithm to compute certificates of (sub)optimality at essentially no extra cost, so bounds on the gap between the objective of the labeling returned by expansion and the optimal objective can be efficiently obtained in practice. 
Unlike our results, these bounds are initialization-dependent, and they only bound the objective value (they do \textit{not} bound the difference from the global minimum itself). Our structural result can be taken as an \emph{explanation} for (i) why these objective bounds tend to be close to 1 regardless of the initialization and (ii) why the returned labelings have small Hamming distance to the optimal labeling: all local minima w.r.t. expansion moves are global minima for some instance within a small perturbation of the input. On practical instances, solutions to small perturbations tend to have near-optimal objective in the original instance and are close in Hamming distance to the original solution.

\subsection{Partial optimality results for $\alpha$-expansion}
A node/label pair $(u,i)$ is a \emph{partially optimal assignment} (henceforth, a \emph{partopt}) if $x^*(u) = i$ for the MAP solution $x^*$.
 Several works have developed fast algorithms for finding provable partopts, i.e. identifying parts of the MAP assignment \citep[e.g.,][]{kovtun2003partial, shekhovtsov2013exact,swoboda2016partial,shekhovtsov2017maximum}. \citet{shekhovtsov2011partial} showed that if Kovtun's procedure outputs a partopt $(u,i)$, then any expansion minimum $x^{\alpha}$ has $x^{\alpha}(u) = i$. Like our result, this gives a guarantee for $\alpha$-expansion that is independent of the algorithm's initialization: expansion always recovers $x^*(u)$ when Kovtun's procedure finds the optimal label at a vertex $u$. However, this result does not explain \emph{when} Kovtun's procedure finds a large number of partopts. In contrast, our results only rely on a structural property of the \emph{instance} itself (that the solutions to perturbations are close to the solutions of the original).
Moreover, our algorithm in Section \ref{sec:verifying} for bounding $\alpha$-expansion's Hamming error is meant to illustrate the tightness of our structural result, not to give a fast method for finding provably partially optimal assignments.

\subsection{Certified algorithms}
Our results are very related to the study of \emph{certified}
algorithms \citep{MakMak20, AngAwaBluChaDan19}. Informally, a
\emph{certified} algorithm is one that returns a global (exact)
solution to a perturbation of the input problem. We prove that
$\alpha$-expansion is a certified algorithm for uniform metric
labeling.
Our algorithm in Section \ref{sec:verifying} for upper bounding $\alpha$-expansion's error could be used to upper bound the error of other certified algorithms.
The fact (proven here) that a popular algorithm with a long track record of empirical success is a certified algorithm suggests that this model could be useful for understanding the empirical performance of algorithms on hard problems. Exact solutions to small perturbations of
the input can be efficiently obtained despite hardness of the original
problem, and these exact solutions are often very close to those of
the original problem in practice.

\section{Expansion always finds a global optimum}
In this section, we give our main theoretical results. Theorem \ref{thm:localglobal} states that every labeling $x$ that is a local minimum  w.r.t. expansion moves is a global minimum (an exact MAP solution) in a perturbed version of the input problem instance.
Theorem \ref{thm:exactpert} then gives a precise characterization of a perturbation in which $x$ is optimal.
The simple structure of these perturbations is useful in the development of our algorithm in Section~\ref{sec:verifying}. We defer both proofs to Appendix~\ref{apdx:main-result}.

\begin{theorem}
\label{thm:localglobal}
Let the labeling $x$ be a local minimum with respect to expansion moves for the instance with objective $\theta$. Let $\mathcal{I}(\theta)$ be the set of $\theta'$ that for some $\gamma \in [1,2]^{|E|}$ satisfy:
\begin{alignat}{2}
    \label{eqn:idefn}
    \theta'_u &= \theta_u \qquad && \forall u \in V\nonumber\\
    \theta'_{uv} &= \gamma_{uv}\theta_{uv} \qquad &&\forall (u,v) \in E
\end{alignat}
Then there exists $\theta' \in \mathcal{I}(\theta)$ for which $x$ is a MAP solution. 
\end{theorem}
The definition of $\mathcal{I}(\theta)$ requires that each $\theta' \in \mathcal{I}(\theta)$ has exactly the same node costs as $\theta$, and that the pairwise potentials $\theta'_{uv} = \gamma_{uv}\theta_{uv}$ for an edge-dependent constant $\gamma_{uv} \in [1,2]$. Theorem \ref{thm:localglobal} says that for each local minimum $x$ to the input instance $\theta$, there exists at least one instance $\theta' \in \mathcal{I}(\theta)$ for which $x$ is a global minimum. The next theorem gives a closed form for one such $\theta'$ in terms of $x$.
\begin{theorem}
\label{thm:exactpert}
Given an instance $\theta$ with edge weights $w_{uv}$ and an expansion minimum $x$ for $\theta$, define perturbed weights $w^x_{uv}$:
\begin{equation}
\label{eqn:wx-defn}
w^x_{uv} = \begin{cases}
    w_{uv} & x(u) \ne x(v)\\
    2w_{uv} & x(u) = x(v),
\end{cases}
\end{equation}
and let 
\begin{equation}
    \label{eqn:thetax-defn}
    \theta^x_{uv}(i,j) = w^x_{uv}\mathbb{I}[i\ne j]
\end{equation}
be the pairwise Potts energies corresponding to the weights $w^x$.
Then $x$ is a global minimum in the instance with objective vector $\theta^x = (\theta_u : u;\ \theta_{uv}^x : uv)$. 
This is the Potts model instance with the same node costs $\theta_u(i)$ as the original instance, but new pairwise energies $\theta_{uv}^x(i,j)$ defined using the perturbed weights $w^x$. Note that $\theta^x \in \mathcal{I}(\theta)$.
Additionally, the LP relaxation \eqref{eqn:local-lp}
is tight on this perturbed instance.
\end{theorem}

Theorem \ref{thm:localglobal} strictly and significantly generalizes Theorem 2 of \citet{LanSonVij18}, since $\mathcal{I}(\theta)$ is the set of $(1,2)$-perturbations of the input instance $\theta$. 
The analysis is similar to that in \citet{LanSonVij18}, but
reinterpreted through the certified algorithm lens of
\citet{MakMak20}. 
The guarantee of Theorem \ref{thm:exactpert} that the LP relaxation \eqref{eqn:local-lp} is tight in the perturbed instance is crucial to our algorithm in Section \ref{sec:verifying}. Theorems \ref{thm:localglobal} and \ref{thm:exactpert} also apply to any iterative algorithm whose set of iterative moves contain the set of expansion moves, such as FastPD.

\section{How ``bad'' are solutions to perturbations?}
\label{sec:verifying}
Theorem \ref{thm:localglobal} guarantees that when $\alpha$-expansion is run on an instance $\theta$, it always returns a MAP solution to some instance $\theta' \in \mathcal{I(\theta)}$.
To evaluate how informative this guarantee is, we need a method to find the ``worst'' solution out of all the solutions to instances in $\mathcal{I}(\theta)$. That is, let 
\begin{equation}
\label{eqn:s-defn}
\mathcal{S}(\theta) = \{x : x \text{ a MAP solution for some } \theta' \in \mathcal{I}(\theta)\}.
\end{equation} Theorem \ref{thm:localglobal} implies all local expansion minima $x$ have $x \in \mathcal{S}(\theta)$. This structural condition is informative for an instance $\theta$ if every solution in $\mathcal{S}(\theta)$ is close to the MAP solution $x^*$ of $\theta$. In that case, our result \emph{explains} why $\alpha$-expansion always performs well. Our hypothesis is that real-world instances should have this property, but we need a method to verify this hypothesis empirically.

In this section, we design an efficient algorithm for upper-bounding the value of any concave function $f(x)$ over $\mathcal{S}(\theta)$. 
For example, $f(x)$ could measure the Hamming distance to $x^*$ or the objective gap of $x$ in the original instance $\theta$.
Note that in these cases, the algorithm must be given $x^*$ to compute $f(x)$, but this need not be true for general $f$. 
For example, in a \emph{learning} scenario, $f(x)$ could measure Hamming distance to the known ground-truth assignment.
Because all local expansion minima are contained in $S(\theta)$ by Theorem \ref{thm:localglobal}, bounds for these quantities give initialization-independent bounds on expansion's performance.

Formally, we want to solve
\begin{align}
\label{eqn:maximization}
\maximize_x \qquad & f(x)\\
\text{subject to} \qquad & x \in \mathcal{S}(\theta).\nonumber
\end{align}
Here $f(x)$ is a concave function that measures the ``badness'' of $x$.
For example, if $x^*$ is a MAP solution to the original instance, we could take $f(x) = \dot{\theta}{x}/\dot{\theta}{x^*}$, the objective gap of $x$. 
Similarly, we can let $f(x)$ be the Hamming distance between $x$ and $x^*$, which can be expressed as:
\[
f(x) = \frac{1}{2n}\left(\sum_{u\in V}\sum_{i\ne x^*(u)} x_u(i) - \sum_{u\in V} x_u(x^*(u)) + n\right),
\]
where we are taking $x^*$ as a labeling and $x$ as a point of $M(G)$. This is an affine function of $x$.
Let $\eta$ be the optimal value of \eqref{eqn:maximization}. Then by Theorem \ref{thm:localglobal}, all local expansion minima $x$ satisfy $f(x) \le \eta$. Solving \eqref{eqn:maximization} thus gives an upper bound on the error of all expansion minima.
For simplicity, and because we use the two functions above in our empirical results, we assume in what follows that $f(x)$ is affine. We can then replace maximization of $f(x)$ with maximization of $\dot{f}{x}$ for some vector $f$. However, our algorithm works for any concave function.

In the rest of this section, we design an efficient algorithm for upper-bounding the optimal value of \eqref{eqn:maximization} by deriving a sequence of equivalent problems, then performing a convex relaxation. In several of our experiments in Section \ref{sec:exps}, we find that the bound obtained by our algorithm is nearly tight.

\begin{theorem}
\label{thm:solvable}
For affine functions $f$, \eqref{eqn:maximization} can be exactly represented by an integer linear program (ILP). Additionally, an upper bound on the optimal value of \eqref{eqn:maximization} can be obtained efficiently using a linear program.
\fxnote{make a note for arbitrary concave $f$?}
\end{theorem}
\begin{proof}
\vspace{-2mm}
As written, \eqref{eqn:maximization} is difficult to optimize because it searches over the set $\mathcal{S}(\theta)$ of MAP solutions to instances $\theta' \in \mathcal{I}(\theta)$. This is not a convex set. First, the following lemma gives a simpler characterization of $\mathcal{S}(\theta)$.
\begin{lemma}
\label{lem:theta-x-enough}
\begin{align*}
    \mathcal{S}(\theta) = \{x : x &\text{ a MAP solution to the instance } \theta^x \\
    &\text{ defined by \eqref{eqn:wx-defn} and \eqref{eqn:thetax-defn}}\}
\end{align*}
\end{lemma}
\begin{proof} We show in Appendix \ref{apdx:main-result} that if $x$ is optimal for any $\theta' \in \mathcal{I}(\theta)$, $x$ is optimal for $\theta^x$. This immediately gives the result.
\end{proof}
So we can rewrite \eqref{eqn:maximization} as:
\begin{align}
\label{eqn:maximization1}
\maximize_x \qquad & f(x)\\
\text{subject to} \qquad & x \text{ a vertex of } M(G),\nonumber\\
& x \text{ optimal in the instance }\nonumber\\
& \text{with objective } \theta^x\nonumber.
\end{align}
The first constraint ensures that $x$ is a valid labeling, and the second constraint ensures (by Lemma \ref{lem:theta-x-enough}) that $x \in \mathcal{S}(\theta)$. Now we focus on simplifying the optimality constraint.
Let $x$ be an optimal labeling in the instance with objective $\theta^x$. By Theorem \ref{thm:exactpert}, for all LP-feasible points $y \in L(G)$, we have that $\dot{\theta^x}{x} \le \dot{\theta^x}{y}$. 
That is, the local LP relaxation is tight on this instance---even though $x$ is an integer solution, its objective value $\dot{\theta^x}{x}$ is as good as that of any fractional solution.
This allows us to rewrite the optimality constraint using the following valid constraint: 
\begin{align*}
\maximize_x \qquad & \dot{f}{x}\\
\text{subject to} \qquad & x \text{ a vertex of } M(G)\\
& \dot{\theta^x}{x} \le \min_{y\in L(G)}\dot{\theta^x}{y},
\end{align*}
Note that if the local LP relaxation were not tight on the instance with objective $\theta^x$ (as guaranteed by Theorem \ref{thm:exactpert}), this constraint would be invalid.
Even with this simplification, the dependence of $\theta^x$ on $x$ makes it unclear whether the new constraint is convex. 
Because $x$ is a vertex of $M(G)$, it only takes values in $\{0,1\}$, so we can rewrite $w^x_{uv}$ from \eqref{eqn:wx-defn} as a linear function of $x$:
\[
w^x_{uv} = w_{uv} + w_{uv}\left(1-\sum_{i\ne j}x_{uv}(i,j)\right).
\]
This is because $\sum_{i\ne j}x_{uv}(i,j)$ is 0 if $x(u) = x(v)$ and 1 otherwise. Then, because $w^x_{uv}$ is a linear function of $x$, $\dot{\theta^x}{y}$ is a linear function of $x$ for each ${y\in L(G)}$. 
Additionally, observe that because $x\in M(G)$, \eqref{eqn:thetax-defn} implies ${\dot{\theta^x}{x} = \dot{\theta}{x}}$: the perturbed objective of $x$ is equal to its original objective (note, however, $y\ne x$ may have $\dot{\theta^x}{y} \ne \dot{\theta}{y}$). 
Using these simplifications, we can solve the following equivalent problem:
\begin{align*}
\maximize_x \qquad & \dot{f}{x}\\
\text{subject to} \qquad & x \text{ a vertex of } M(G)\\
& \dot{\theta}{x} \le \min_{y\in L(G)}\dot{\theta^x}{y}.
\end{align*}
Because we have shown how to re-write $\dot{\theta^x}{y}$ as a linear function of $x$ and removed $\theta^x$ from the left-hand-side, the second constraint is convex. 
However, two barriers remain to solving this problem efficiently: (i) the optimality constraint $\dot{\theta}{x} \le \min_{y\in L(G)}\dot{\theta^x}{y}$ is not in a convenient form, and (ii) the first constraint is not convex.
We address (i) first.

For ease of notation, define $A$ and $b$ so that the local polytope ${L(G) = \{x | Ax = b,\ x\ge 0\}}$.
Because strong duality holds for the local LP relaxation in the instance with objective $\theta^x$, we know that 
\[
\min_{y:Ay=b,y\ge 0} \dot{\theta^x}{y} = \max_{\nu: A^T \nu \le \theta^x} \dot{b}{\nu}.
\]
Indeed, if there exists any feasible $y,\nu$ pair for which $\dot{\theta^x}{y} = \dot{b}{\nu}$, $y$ and $\nu$ are optimal primal and dual solutions, respectively.
We want to enforce the constraint that $x$ is primal-optimal in the instance with objective $\theta^x$, which is the case if and only if there exists a dual-feasible $\nu$ with $\dot{\theta^x}{x} = \dot{\theta}{x} = \dot{b}{\nu}$.
So we can rewrite the problem as
\begin{align*}
\maximize_{x,\nu} \qquad & \dot{f}{x}\\
\text{subject to} \qquad & x \text{ a vertex of } M(G)\\
& \dot{\theta}{x} = \dot{b}{\nu}\\
& A^T\nu \le \theta^x.
\end{align*}
Because $\theta^x$ is a linear function of $x$, the latter two constraints are linear in $x$ and $\nu$.
Together with linearizing $\theta^x$ and noting $\dot{\theta^x}{x} = \dot{\theta}{x}$, this primal-dual trick allowed us to encode the second constraint of \eqref{eqn:maximization1} as two sets of linear constraints. 
This trick heavily relies on the guarantee from Theorem \ref{thm:exactpert} that the local LP is tight on the instance with objective $\theta^x$.

\begin{figure*}[t]
    \centering
    \begin{subfigure}{0.25\textwidth}
    \includegraphics[width=1.0\textwidth]{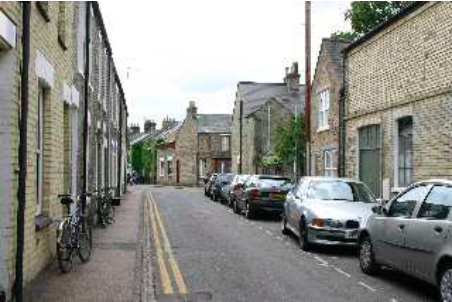}    
    \end{subfigure}%
    \begin{subfigure}{0.25\textwidth}
    \includegraphics[width=1.0\textwidth]{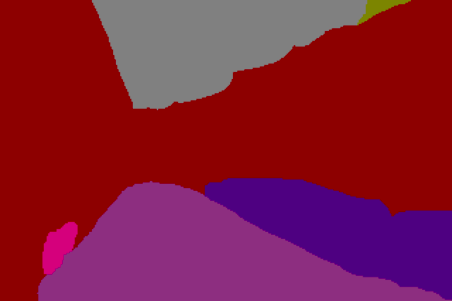}    
    \end{subfigure}%
    \begin{subfigure}{0.25\textwidth}
    \includegraphics[width=1.0\textwidth]{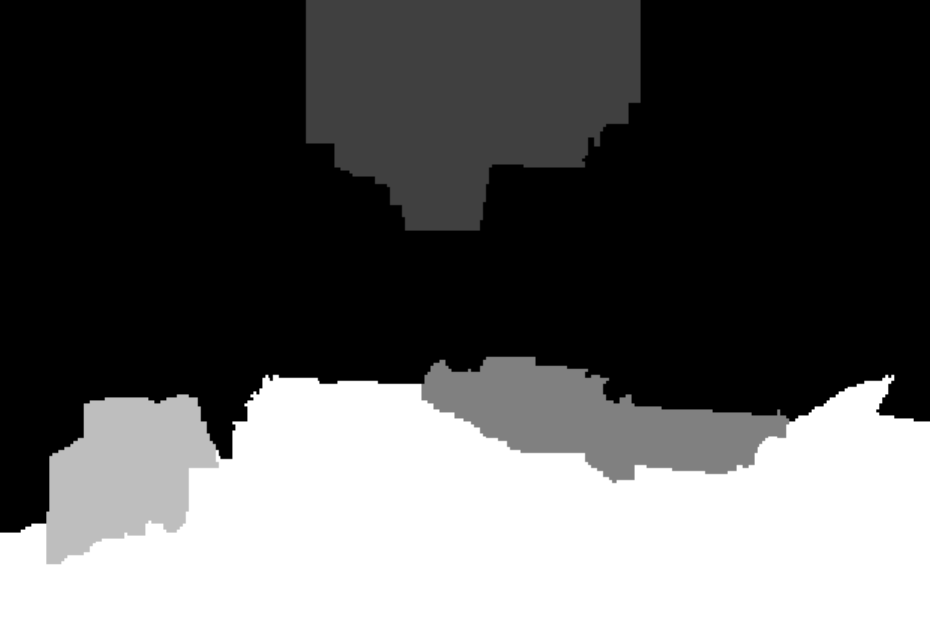}    
    \end{subfigure}%
    \begin{subfigure}{0.25\textwidth}
    \includegraphics[width=1.0\textwidth]{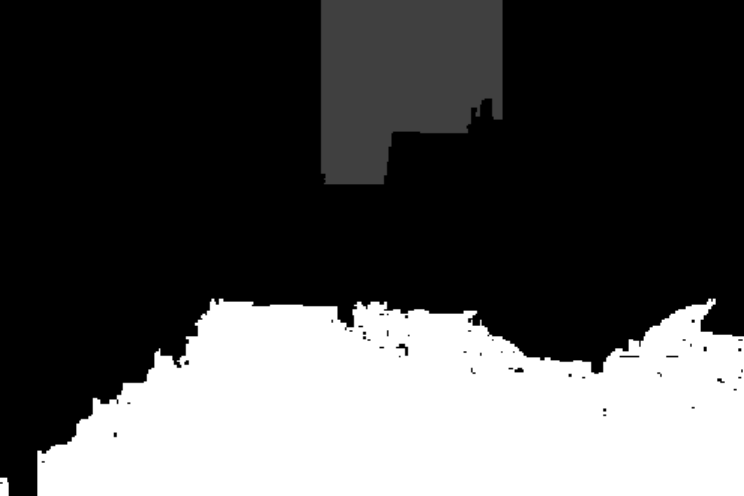}
    \end{subfigure}\\
    \begin{subfigure}{0.25\textwidth}
    \includegraphics[width=1.0\textwidth]{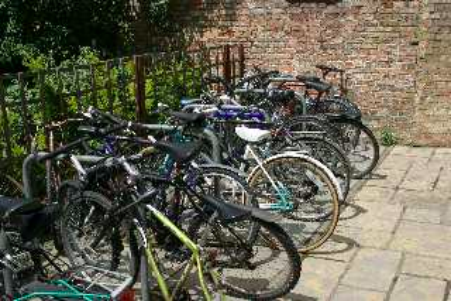}    
    \end{subfigure}%
    \begin{subfigure}{0.25\textwidth}
    \includegraphics[width=1.0\textwidth]{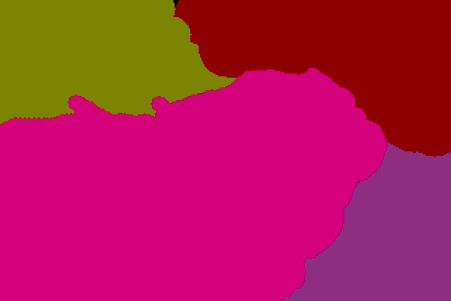}    
    \end{subfigure}%
    \begin{subfigure}{0.25\textwidth}
    \includegraphics[width=1.0\textwidth]{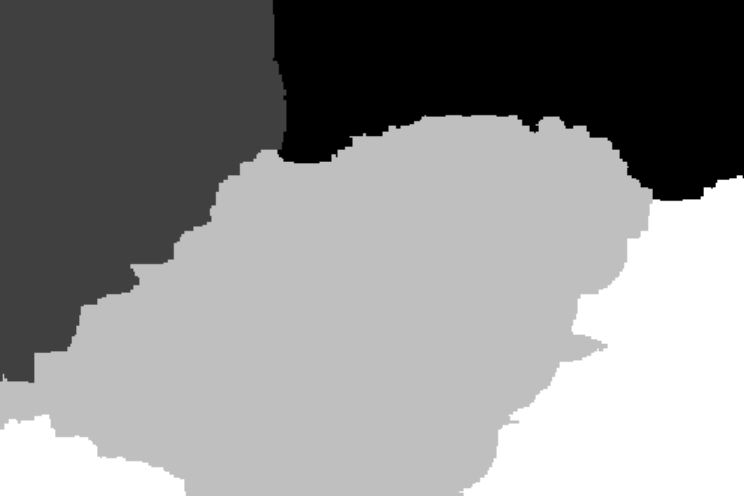}    
    \end{subfigure}%
    \begin{subfigure}{0.25\textwidth}
    \includegraphics[width=1.0\textwidth]{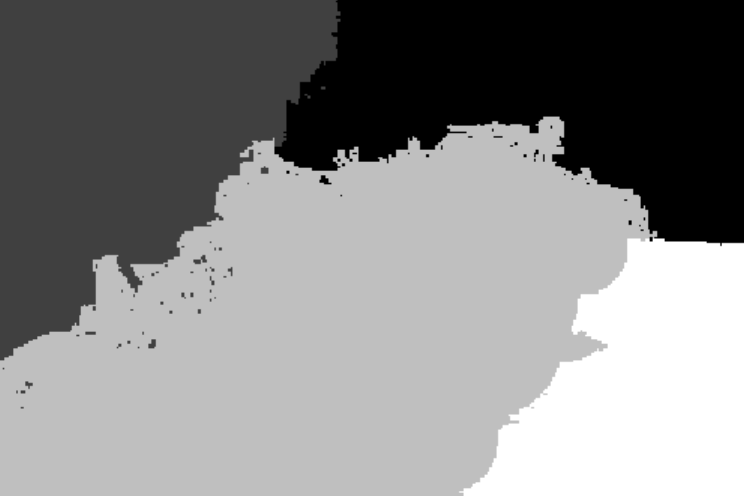}
    \end{subfigure}
    \begin{subfigure}{0.25\textwidth}
    \includegraphics[width=1.0\textwidth]{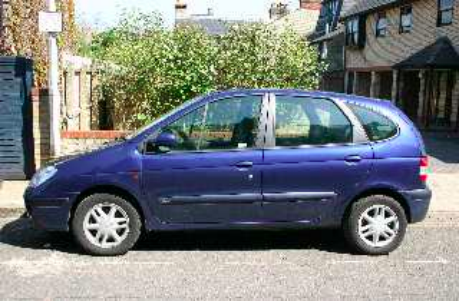}    
    \end{subfigure}%
    \begin{subfigure}{0.25\textwidth}
    \includegraphics[width=1.0\textwidth]{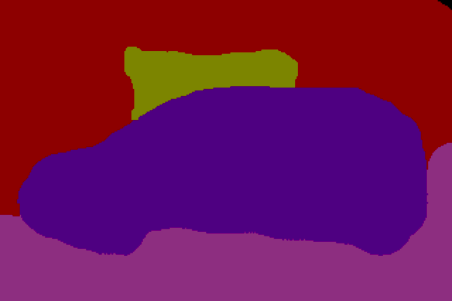}    
    \end{subfigure}%
    \begin{subfigure}{0.25\textwidth}
    \includegraphics[width=1.0\textwidth]{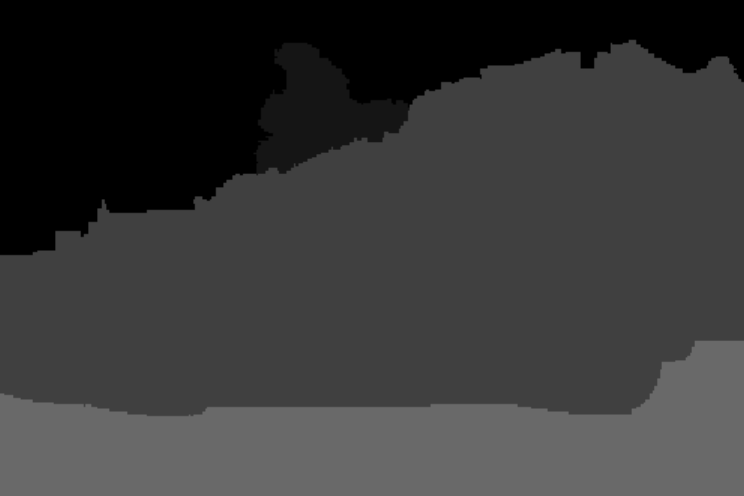}    
    \end{subfigure}%
    \begin{subfigure}{0.25\textwidth}
    \includegraphics[width=1.0\textwidth]{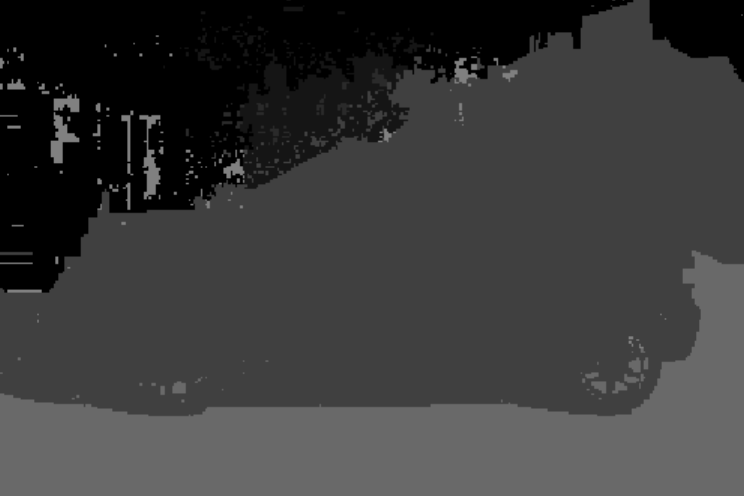}
    \end{subfigure}    
    \caption{Left column: original image; Center-left: ground-truth segmentation map; Center-right: exact MAP solution $x^*$ in the observed instance; Right: a local expansion minimum that nearly achieves our upper bound on the Hamming error. Rows: {\tt road}, {\tt bikes}, {\tt car}. On these instances, our theoretical result guarantees that the Hamming error of \emph{any} local expansion minimum is at most 17\%, 14\%, and 8\%, respectively. The local expansion minima in the rightmost column have Hamming error of 11\%, 8\%, and 7\% of the nodes, respectively. Our theoretical result implies that these local minima are almost the ``worst possible'' w.r.t. Hamming error. These ``bad'' expansion minima were found by initializing the $\alpha$-expansion algorithm with (a rounded version of) the labeling $x$ output by \eqref{eqn:relaxed-alg}.}
    \label{fig:objseg_bound}
    \vspace{-1em}
\end{figure*}

The only remaining issue is the first constraint, that $x$ is a vertex of $M(G)$.
We saw in Section \ref{sec:prelim} how to encode the vertices of $M(G)$ using linear and integrality constraints,
so we can rewrite the above problem as the ILP:
\begin{align}
\label{eqn:exact-alg}
\maximize_{x,\nu} \qquad & \dot{f}{x}\\
\text{subject to} \qquad & x \in L(G)\nonumber\\
& x_{u}(i) \in \{0,1\}\nonumber\\
& x_{uv}(i,j) \in \{0,1\}\nonumber\\
& \dot{\theta}{x} = \dot{b}{\nu}\nonumber\\
& A^T\nu \le \theta^x.\nonumber
\end{align}
Unfortunately, this ILP is too large for off-the-shelf ILP solvers to handle in practice. 
Instead, we relax this exact formulation to obtain upper bounds.

In particular, we iteratively solve the following optimization problem:
\begin{align}
\label{eqn:relaxed-alg}
\maximize_{x,\nu} \qquad & \dot{f}{x}\\
\text{subject to} \qquad & x \in K_t\nonumber\\
& \dot{\theta}{x} = \dot{b}{\nu}\nonumber\\
& A^T\nu \le \theta^x\nonumber,
\end{align}
where $M(G) \subset K_t$ for all $t$, and $K_t \subset K_{t-1}$. 
We start with $K_0 = L(G)$, then use the ``cycle constraints'' from \citet{sontag2008new} to go from $K_t$ to $K_{t+1}$.
Violated cycle constraints can be found efficiently by computing shortest paths in an auxiliary graph that depends on the solution $x^t$ to this program. 
Even if we could efficiently represent the constraint that $x\in M(G)$, this approach would still be a relaxation of the ILP formulation, because the optimal $x$ may not be attained at a vertex of $M(G)$.
However, this relaxation is nearly tight on several of our empirical examples. 
The exact ILP formulation \eqref{eqn:exact-alg} and its relaxation \eqref{eqn:relaxed-alg} give both claims of Theorem \ref{thm:solvable}.
\end{proof}

There is also a simpler approach to upper-bounding the optimal value of \eqref{eqn:maximization} for affine $f$ based on Theorem \ref{thm:boykov}, the original approximation guarantee for $\alpha$-expansion. That result guarantees that any expansion minimum satisfies $\dot{\theta}{x} \le \dot{\theta}{x^*} + \sum_{uv}\theta_{uv}(x^*(u),x^*(v)).$ Therefore, we can upper bound \eqref{eqn:maximization} with the ILP:
\begin{align}
\label{eqn:naive-alg}
\maximize_{x} \qquad & \dot{f}{x}\\
\text{subject to} \qquad & x \text{ a vertex of } M(G)\nonumber\\
& \dot{\theta}{x} \le \dot{\theta}{x^*} + \sum_{uv}\theta_{uv}(x^*(u),x^*(v)).\nonumber
\end{align}
Like \eqref{eqn:exact-alg}, this is an ILP. 
We refer to this as the \emph{\naive} bound, since it comes directly from the approximation guarantee for $\alpha$-expansion.
In the next section, we compare \eqref{eqn:relaxed-alg} to \eqref{eqn:naive-alg} on real-world instances, and find that our bound \eqref{eqn:relaxed-alg} is much tighter. 
Intuitively, our bound carefully tries to enforce that the optimization variable $x$ is an \emph{optimal point} in some instance, whereas the \naive~bound may allow for feasible points $x$ that are not optimal in any instance.

\section{Numerical results}
\label{sec:exps}
\begin{table*}[t]
    \centering
    \caption{Results of our bound on six real instances.}
    \begin{tabular}{lcccc}
         Instance & Obj. bd. (ours) & Obj. bd. (\naive) & Ham. err. bd. (ours)  & Ham. bd. (\naive)\\
         \toprule
         ${\tt tsukuba}$ & 1.213 & 1.228 & 0.290 & 0.821\\
         ${\tt venus}$ & 1.199 & 1.268 & 0.375 & 0.703\\
         ${\tt plastic}$ & 1.073 & 1.095 & 0.373 & 0.779\\
         \midrule
         ${\tt road}$ & 1.031 & 1.036 & 0.171 (0.114) & 0.256\\
         ${\tt bikes}$ & 1.027 & 1.030 & 0.146 (0.082) & 0.229\\
         ${\tt car}$ & 1.019 & 1.047 & 0.081 (0.074) & 0.225\\         
         \bottomrule
    \end{tabular}
    \vspace{1em}
    \caption*{Results on six MAP inference instances from computer vision: 3 stereo vision (top) and 3 object segmentation (bottom). Our bounds on the objective gap and Hamming error are obtained by iteratively running \eqref{eqn:relaxed-alg}. The ``\naive'' bounds are obtained by using \eqref{eqn:naive-alg}. Our procedure results in slightly tighter objective gap bounds and much tighter Hamming error bounds. For the object segmentation instances, lower bounds on the Hamming error of local expansion minima are shown in parentheses. That is, there exist local expansion minima with the Hamming error displayed in parentheses. These minima are shown in Figure \ref{fig:objseg_bound}, and were found by running $\alpha$-expansion initialized with the output of \eqref{eqn:relaxed-alg}. Our Hamming error bound implies that these are almost the ``worst possible'' expansion minima w.r.t. Hamming error. For example, on the {\tt car} instance, our bound guarantees that \emph{any} local expansion minimum agrees with the MAP solution on at least 91.9\% of the vertices, and we have found a local minimum that agrees with the MAP solution on just 92.6\% of the vertices.}
    \label{tbl:boundtable}
    \vspace{-0.5em}
\end{table*}

In this section, we run \eqref{eqn:relaxed-alg} on several real-world MAP inference instances to evaluate the tightness of bounds derived from our structural condition (Theorem \ref{thm:localglobal}). 
Theorem \ref{thm:localglobal} guarantees that all local expansion minima $x$ for instance $\theta$ are contained in $\mathcal{S}(\theta)$, the set of exact solutions to certain perturbations of the input problem $\theta$.
If we upper bound the Hamming distance to $x^*$ and the objective gap $\dot{\theta}{x}/\dot{\theta}{x^*}$ over $\mathcal{S}(\theta)$, we obtain upper bounds on the Hamming recovery and objective gap that apply to \emph{all} solutions that can possibly be returned by $\alpha$-expansion.
These ``problem-dependent worst-case'' bounds hold for every possible initial labeling and every possible update order in Algorithm \ref{alg:expansion}.

Broadly, we find that the real-world examples we study are not pathological: global optima to perturbed instances tend to be quite close to global optima of the original instance.
Together with Theorem \ref{thm:localglobal}, this implies that these instances have no spurious local minima w.r.t. expansion moves.

We study two types of instances: first, a stereo vision problem, where the weights $w$ and costs $\theta_u(i)$ are set ``by hand'' according to the model from \citet{tappen2003comparison}. 
Given two images taken from slightly offset locations, the goal is to estimate the depth of every pixel in one of the images. 
This can be done by estimating, for each pixel, the disparity between the two images, since the depth is inversely proportional to the disparity. 
In the \citet{tappen2003comparison} model, the node costs are set using the sampling-invariant technique from \citet{birchfield1998pixel}, and the weights $w_{uv}$ are set as:
\begin{equation*}
  w_{uv} = \begin{cases}
    P \times s & |I(u) - I(v)| < T\\
    s & \text{otherwise},
  \end{cases}
\end{equation*}
where $P,T,$ and $s$ are the parameters of the model, and $I(u)$ is the intensity of pixel $u$ in one of the input images to the stereo problem. These edge weights charge more for separating pixels with similar intensities, since nearby pixels with similar intensities are likely to be at the same depth. We also study object segmentation instances, where the weights $w$ and costs $\theta_u$ are learned from data. In this problem, the goal is to assign a label to each pixel that represents the object to which that pixel belongs.
For these instances, we use the models from \citet{alahari2010dynamic}.
We include the full details of both models in Appendix \ref{apdx:model-details}.

Table \ref{tbl:boundtable} shows the results of running several rounds of \eqref{eqn:relaxed-alg} on six of these instances. 
For each instance, we compare against the \naive~objective bound $\dot{\theta}{x^*} + \sum_{uv}\theta_{uv}(x^*(u),x^*(v))$ obtained from the original proof of $\alpha$-expansion's approximation guarantee, and against the \naive~Hamming bound obtained by solving \eqref{eqn:naive-alg}. 
We used Gurobi \citep{gurobi} to run the iterations of \eqref{eqn:relaxed-alg} and to solve the ILP \eqref{eqn:naive-alg}. We added cycle inequalities using the $k$-projection graph \citep{sontag2008new}, adding several violated inequalities per iteration. We ran between 10 and 20 iterations of \eqref{eqn:relaxed-alg} for each experiment. Tightening using the cycle inequalities was beneficial in practice. For example, it improved our Hamming error bound on the {\tt tsukuba} instance from $0.38$ to $0.29$.

Compared to \eqref{eqn:naive-alg}, our procedure results in slightly tighter objective bounds and much tighter Hamming bounds on these instances. For example, on the {\tt car} instance, our bound certifies that \emph{all} local minima w.r.t. expansion moves must agree with the MAP solution $x^*$ on at least 91.9\% of the nodes. Moreover, there exists an expansion minimum for this instance that agrees on only 92.6\% of the vertices, which nearly matches our bound. This ``worst-case'' expansion minimum is shown in Figure~\ref{fig:objseg_bound}.

\section{Conclusion}
We have shown that graph cuts algorithms, such as $\alpha$-expansion and FastPD, take advantage of special structure in real-world problem instances with Potts potentials.
Our empirical results show that the solutions (the global energy minima) to small perturbations of the input are often very close to the solutions of the original instance. 
Our theoretical result states that all local minima w.r.t. expansion moves are global minima in such perturbations.
Taken together, these two results imply that there are no spurious local minima w.r.t. expansion moves in practice.
This gives a new theoretical explanation for the good performance of graph cuts algorithms in the wild.
Moreover, our structural result could have practical consequences for learning Markov random fields. To ensure $\alpha$-expansion performs well on an instance, one could add a regularization term during learning that encourages the solutions to small perturbations $\mathcal{I}(\theta)$ of the instance to be close to the solution of the original.

\clearpage
\clearpage
\section*{Acknowledgments}
The authors thank Chandler Squires for his helpful feedback on drafts of this paper and an anonymous reviewer for pointing us to \citet{shekhovtsov2011partial}. This work was supported by NSF AitF awards CCF-1637585 and CCF-1723344.
\bibliography{graphcuts}

\begin{thebibliography}{30}
\providecommand{\natexlab}[1]{#1}
\providecommand{\url}[1]{\texttt{#1}}
\expandafter\ifx\csname urlstyle\endcsname\relax
  \providecommand{\doi}[1]{doi: #1}\else
  \providecommand{\doi}{doi: \begingroup \urlstyle{rm}\Url}\fi

\bibitem[Alahari et~al.(2010)Alahari, Kohli, and Torr]{alahari2010dynamic}
Alahari, K., Kohli, P., and Torr, P.~H.
\newblock Dynamic hybrid algorithms for map inference in discrete mrfs.
\newblock \emph{IEEE Transactions on Pattern Analysis and Machine
  Intelligence}, 32\penalty0 (10):\penalty0 1846--1857, 2010.

\bibitem[Angelidakis et~al.(2017)Angelidakis, Makarychev, and
  Makarychev]{AngMakMak17}
Angelidakis, H., Makarychev, K., and Makarychev, Y.
\newblock Algorithms for stable and perturbation-resilient problems.
\newblock In \emph{Proceedings of the 49th Annual ACM SIGACT Symposium on
  Theory of Computing}, pp.\  438--451, 2017.

\bibitem[Angelidakis et~al.(2019)Angelidakis, Awasthi, Blum, Chatziafratis, and
  Dan]{AngAwaBluChaDan19}
Angelidakis, H., Awasthi, P., Blum, A., Chatziafratis, V., and Dan, C.
\newblock Bilu-linial stability, certified algorithms and the independent set
  problem.
\newblock In \emph{27th Annual European Symposium on Algorithms, ESA 2019},
  pp.\ ~7. Schloss Dagstuhl-Leibniz-Zentrum fur Informatik GmbH, Dagstuhl
  Publishing, 2019.

\bibitem[Archer et~al.(2004)Archer, Fakcharoenphol, Harrelson, Krauthgamer,
  Talwar, and Tardos]{archer2004approximate}
Archer, A., Fakcharoenphol, J., Harrelson, C., Krauthgamer, R., Talwar, K., and
  Tardos, {\'E}.
\newblock Approximate classification via earthmover metrics.
\newblock In \emph{Proceedings of the fifteenth annual ACM-SIAM symposium on
  Discrete algorithms}, pp.\  1079--1087. Society for Industrial and Applied
  Mathematics, 2004.

\bibitem[Birchfield \& Tomasi(1998)Birchfield and Tomasi]{birchfield1998pixel}
Birchfield, S. and Tomasi, C.
\newblock A pixel dissimilarity measure that is insensitive to image sampling.
\newblock \emph{IEEE Transactions on Pattern Analysis and Machine
  Intelligence}, 20\penalty0 (4):\penalty0 401--406, 1998.

\bibitem[Boykov et~al.(2001)Boykov, Veksler, and Zabih]{BoyVekZab01}
Boykov, Y., Veksler, O., and Zabih, R.
\newblock Fast approximate energy minimization via graph cuts.
\newblock \emph{IEEE Transactions on pattern analysis and machine
  intelligence}, 23\penalty0 (11):\penalty0 1222--1239, 2001.

\bibitem[Dahlhaus et~al.(1992)Dahlhaus, Johnson, Papadimitriou, Seymour, and
  Yannakakis]{dahlhaus1992complexity}
Dahlhaus, E., Johnson, D.~S., Papadimitriou, C.~H., Seymour, P.~D., and
  Yannakakis, M.
\newblock The complexity of multiway cuts.
\newblock In \emph{Proceedings of the twenty-fourth annual ACM symposium on
  Theory of computing}, pp.\  241--251, 1992.

\bibitem[Geman \& Geman(1984)Geman and Geman]{geman1984stochastic}
Geman, S. and Geman, D.
\newblock Stochastic relaxation, gibbs distributions, and the bayesian
  restoration of images.
\newblock \emph{IEEE Transactions on pattern analysis and machine
  intelligence}, \penalty0 (6):\penalty0 721--741, 1984.

\bibitem[Globerson \& Jaakkola(2008)Globerson and
  Jaakkola]{globerson2008fixing}
Globerson, A. and Jaakkola, T.~S.
\newblock Fixing max-product: Convergent message passing algorithms for map
  lp-relaxations.
\newblock In \emph{Advances in neural information processing systems}, pp.\
  553--560, 2008.

\bibitem[Gurobi~Optimization(2020)]{gurobi}
Gurobi~Optimization, L.
\newblock Gurobi optimizer reference manual, 2020.
\newblock URL \url{http://www.gurobi.com}.

\bibitem[Kappes et~al.(2015)Kappes, Andres, Hamprecht, Schn{\"o}rr, Nowozin,
  Batra, Kim, Kausler, Kr{\"o}ger, Lellmann, et~al.]{kappes2015comparative}
Kappes, J.~H., Andres, B., Hamprecht, F.~A., Schn{\"o}rr, C., Nowozin, S.,
  Batra, D., Kim, S., Kausler, B.~X., Kr{\"o}ger, T., Lellmann, J., et~al.
\newblock A comparative study of modern inference techniques for structured
  discrete energy minimization problems.
\newblock \emph{International Journal of Computer Vision}, 115\penalty0
  (2):\penalty0 155--184, 2015.

\bibitem[Kleinberg \& Tardos(2002)Kleinberg and Tardos]{KleinbergTardos02}
Kleinberg, J. and Tardos, E.
\newblock Approximation algorithms for classification problems with pairwise
  relationships: Metric labeling and markov random fields.
\newblock \emph{Journal of the ACM (JACM)}, 49\penalty0 (5):\penalty0 616--639,
  2002.

\bibitem[Komodakis et~al.(2007)Komodakis, Tziritas, and
  Paragios]{komodakis2007fast}
Komodakis, N., Tziritas, G., and Paragios, N.
\newblock Fast, approximately optimal solutions for single and dynamic mrfs.
\newblock In \emph{2007 IEEE Conference on Computer Vision and Pattern
  Recognition}, pp.\  1--8. IEEE, 2007.

\bibitem[Kovtun(2003)]{kovtun2003partial}
Kovtun, I.
\newblock Partial optimal labeling search for a np-hard subclass of (max,+)
  problems.
\newblock In \emph{Joint Pattern Recognition Symposium}, pp.\  402--409.
  Springer, 2003.

\bibitem[Lang et~al.(2018)Lang, Sontag, and Vijayaraghavan]{LanSonVij18}
Lang, H., Sontag, D., and Vijayaraghavan, A.
\newblock Optimality of approximate inference algorithms on stable instances.
\newblock In \emph{International Conference on Artificial Intelligence and
  Statistics}, pp.\  1157--1166, 2018.

\bibitem[Lang et~al.(2019)Lang, Sontag, and Vijayaraghavan]{LanSonVij19}
Lang, H., Sontag, D., and Vijayaraghavan, A.
\newblock Block stability for map inference.
\newblock In \emph{The 22nd International Conference on Artificial Intelligence
  and Statistics}, pp.\  216--225, 2019.

\bibitem[Lang et~al.(2021)Lang, Reddy, Sontag, and
  Vijayaraghavan]{LanRedSonVij21}
Lang, H., Reddy, A., Sontag, D., and Vijayaraghavan, A.
\newblock Beyond perturbation stability: Lp recovery guarantees for map
  inference on noisy stable instances.
\newblock In \emph{The 24th International Conference on Artificial Intelligence
  and Statistics}. PMLR, 2021.

\bibitem[Makarychev \& Makarychev(2020)Makarychev and Makarychev]{MakMak20}
Makarychev, K. and Makarychev, Y.
\newblock Certified algorithms: Worst-case analysis and beyond.
\newblock In \emph{11th Innovations in Theoretical Computer Science Conference
  (ITCS 2020)}. Schloss Dagstuhl-Leibniz-Zentrum f{\"u}r Informatik, 2020.

\bibitem[Scharstein et~al.(2014)Scharstein, Hirschm{\"u}ller, Kitajima,
  Krathwohl, Ne{\v{s}}i{\'c}, Wang, and Westling]{scharstein2014high}
Scharstein, D., Hirschm{\"u}ller, H., Kitajima, Y., Krathwohl, G.,
  Ne{\v{s}}i{\'c}, N., Wang, X., and Westling, P.
\newblock High-resolution stereo datasets with subpixel-accurate ground truth.
\newblock In \emph{German conference on pattern recognition}, pp.\  31--42.
  Springer, 2014.

\bibitem[Shekhovtsov(2013)]{shekhovtsov2013exact}
Shekhovtsov, A.
\newblock Exact and partial energy minimization in computer vision.
\newblock 2013.

\bibitem[Shekhovtsov \& Hlavac(2011)Shekhovtsov and
  Hlavac]{shekhovtsov2011partial}
Shekhovtsov, A. and Hlavac, V.
\newblock On partial optimality by auxiliary submodular problems.
\newblock \emph{Control Systems and Computers}, \penalty0 (2), 2011.

\bibitem[Shekhovtsov et~al.(2017)Shekhovtsov, Swoboda, and
  Savchynskyy]{shekhovtsov2017maximum}
Shekhovtsov, A., Swoboda, P., and Savchynskyy, B.
\newblock Maximum persistency via iterative relaxed inference in graphical
  models.
\newblock \emph{IEEE Transactions on Pattern Analysis and Machine
  Intelligence}, 40\penalty0 (7):\penalty0 1668--1682, 2017.

\bibitem[Shotton et~al.(2006)Shotton, Winn, Rother, and
  Criminisi]{shotton2006textonboost}
Shotton, J., Winn, J., Rother, C., and Criminisi, A.
\newblock Textonboost: Joint appearance, shape and context modeling for
  multi-class object recognition and segmentation.
\newblock In \emph{European conference on computer vision}, pp.\  1--15.
  Springer, 2006.

\bibitem[Sontag \& Jaakkola(2008)Sontag and Jaakkola]{sontag2008new}
Sontag, D. and Jaakkola, T.~S.
\newblock New outer bounds on the marginal polytope.
\newblock In \emph{Advances in Neural Information Processing Systems}, pp.\
  1393--1400, 2008.

\bibitem[Sontag(2010)]{sontag2010approximate}
Sontag, D.~A.
\newblock \emph{Approximate inference in graphical models using LP
  relaxations}.
\newblock PhD thesis, Massachusetts Institute of Technology, 2010.

\bibitem[Swoboda et~al.(2016)Swoboda, Shekhovtsov, Kappes, Schn{\"o}rr, and
  Savchynskyy]{swoboda2016partial}
Swoboda, P., Shekhovtsov, A., Kappes, J.~H., Schn{\"o}rr, C., and Savchynskyy,
  B.
\newblock Partial optimality by pruning for map-inference with general
  graphical models.
\newblock \emph{IEEE Transactions on Pattern Analysis and Machine
  Intelligence}, 38\penalty0 (7), 2016.

\bibitem[Szeliski et~al.(2008)Szeliski, Zabih, Scharstein, Veksler, Kolmogorov,
  Agarwala, Tappen, and Rother]{szeliski2008comparative}
Szeliski, R., Zabih, R., Scharstein, D., Veksler, O., Kolmogorov, V., Agarwala,
  A., Tappen, M., and Rother, C.
\newblock A comparative study of energy minimization methods for markov random
  fields with smoothness-based priors.
\newblock \emph{IEEE transactions on pattern analysis and machine
  intelligence}, 30\penalty0 (6):\penalty0 1068--1080, 2008.

\bibitem[{Tappen} \& {Freeman}(2003){Tappen} and
  {Freeman}]{tappen2003comparison}
{Tappen} and {Freeman}.
\newblock Comparison of graph cuts with belief propagation for stereo, using
  identical mrf parameters.
\newblock In \emph{Proceedings Ninth IEEE International Conference on Computer
  Vision}, pp.\  900--906 vol.2, 2003.

\bibitem[Wainwright \& Jordan(2008)Wainwright and
  Jordan]{wainwright2008graphical}
Wainwright, M.~J. and Jordan, M.~I.
\newblock \emph{Graphical models, exponential families, and variational
  inference}.
\newblock Now Publishers Inc, 2008.

\bibitem[Zheng et~al.(2015)Zheng, Jayasumana, Romera-Paredes, Vineet, Su, Du,
  Huang, and Torr]{zheng2015conditional}
Zheng, S., Jayasumana, S., Romera-Paredes, B., Vineet, V., Su, Z., Du, D.,
  Huang, C., and Torr, P.~H.
\newblock Conditional random fields as recurrent neural networks.
\newblock In \emph{Proceedings of the IEEE international conference on computer
  vision}, pp.\  1529--1537, 2015.

\end{thebibliography}
\bibliographystyle{icml2021}
\onecolumn
\icmlsupptitle{Graph cuts always find a global optimum for Potts models (with a catch): supplementary material}

\appendix

\section{Proof of Theorem \ref{thm:exactpert}}
\label{apdx:main-result}
In this section, we give the full proof of Theorem \ref{thm:exactpert}, restated here. Theorem \ref{thm:localglobal} is then a straightforward corollary of Theorem \ref{thm:exactpert} (Theorem \ref{thm:exactpert} is essentially the constructive version of Theorem \ref{thm:localglobal}).
\begin{theorem*}
Consider an input instance $\theta$ with Potts pairwise potentials and weights $w$, and let the labeling $x$ be a local minimum for $\theta$ with respect to expansion moves. Define perturbed weights ${w^x: E\to \mathbb{R}_+}$ as
\begin{equation}
\label{eqn:wx-defn-apdx}
w^x_{uv} = \begin{cases}
    w_{uv} & x(u) \ne x(v)\\
    2w_{uv} & x(u) = x(v),
\end{cases}
\end{equation}
and let 
\begin{equation}
    \label{eqn:thetax-defn-apdx}
    \theta^x_{uv}(i,j) = w^x_{uv}\mathbb{I}[i\ne j]
\end{equation}
be the pairwise Potts energies corresponding to the weights $w^x$.
Then $x$ is a global minimum in the instance with objective vector $\theta^x = (\theta_u : u;\ \theta_{uv}^x : uv)$. 
This is the Potts model instance with the same node costs $\theta_u(i)$ as the original instance, but new pairwise energies $\theta_{uv}^x(i,j)$ defined using the perturbed weights $w^x$.
Additionally, the LP relaxation \eqref{eqn:local-lp}
is tight on this perturbed instance.
\end{theorem*}

\begin{proof}
  Let $x$ be any labeling of $G$. We show that there exists an expansion $x^{\alpha}$ of $x$ such that for some $\epsilon > 0$:
\begin{equation}
    \label{eqn:proof-ineq-apdx}
      \langle \theta, x - x^{\alpha} \rangle \ge \epsilon\left(\langle \theta^x, x \rangle - \min_{y\in L(G)} \langle \theta^x,  y\rangle\right).
\end{equation}
This implies that as long as $\langle \theta^x, x \rangle - \min_{y\in L(G)} \langle \theta^x,  y\rangle$ is positive, there exists an expansion move with strictly better objective than $x$. The right-hand-side of \eqref{eqn:proof-ineq-apdx} is always nonnegative, because $x\in L(G)$. 
Therefore, $x$ can only be a local minimum w.r.t. expansion moves if $\langle \theta^x, x \rangle = \min_{y\in L(G)} \langle \theta^x,  y\rangle$. 
Every labeling of $G$ corresponds to a point in $L(G)$, since $M(G) \subset L(G)$, so if $\dot{\theta^x}{x} = \min_{y\in L(G)} \langle \theta^x,  y\rangle$, $x$ must be an optimal labeling in the instance with objective $\theta^x$. 
This equality also implies that a vertex of $M(G)$ attains the optimal objective value for \eqref{eqn:local-lp}, which is the definition of ``tightness'' on an instance. 
So \eqref{eqn:proof-ineq-apdx} dgives both claims of the theorem.

  Let $y' \in \argmin_{y\in L(G)} \langle \theta^x, y\rangle$ be an LP solution to the perturbed instance. To show \eqref{eqn:proof-ineq-apdx}, we design a rounding algorithm $R$ that takes $y'$ and $x$ as input and outputs an expansion move $x^{\alpha}$ of $x$. We show that $R$ satisfies
  \begin{equation}
      \label{eqn:rounding-ineq-apdx}
        \E[\langle\theta, x - R(x,y')\rangle] \ge \epsilon(\langle\theta^x, x - y'\rangle),
  \end{equation}
  which proves \eqref{eqn:proof-ineq-apdx} because it implies there exists some $x^{\alpha}$ in the support of $R(x,y')$ that attains \eqref{eqn:proof-ineq-apdx}.
  
\begin{algorithm}[H]
   \caption{$R(x,y')$}
   \label{alg:rounding_algorithm}
\begin{algorithmic}[1]
   \STATE Fix $0 < \epsilon < 1/k$.
   \STATE Set $x' = \epsilon y' + (1-\epsilon)x$.
   \STATE Choose $i \in \{1,\dots,k\}$ uniformly at random.
   \STATE Choose $r \in (0,1/k)$ uniformly at random.
   \STATE Initialize labeling $x^{\alpha}: V\to [k]$.
   \FOR{each $u \in V$}
        \IF {$x'_u(i) > r$}
            \STATE Set $x^{\alpha}(u) = i$.
        \ELSE
            \STATE Set $x^{\alpha}(u) = x(u)$
        \ENDIF 
   \ENDFOR
   \STATE \textbf{Return} $x^{\alpha}$
\end{algorithmic}
\end{algorithm}

\begin{lemma}[Rounding guarantees]
\label{lemma:rounding-guar}
The labeling $x^{\alpha}$ output by Algorithm \ref{alg:rounding_algorithm} is an expansion of $x$, and it satisfies the following guarantees:
\begin{align*}
    \P[x^{\alpha}(u) = i] &= x'_u(i) &\forall\ u\in V, i\in [k]\\
    \P[x^{\alpha}(u) \ne x^{\alpha}(v)] &\le 2d(u,v) &\forall\ (u,v) \in E : x(u) = x(v)\\
    \P[x^{\alpha}(u) = x^{\alpha}(v)] &= (1-d(u,v)) &\forall\ (u,v) \in E : x(u) \ne x(v),
\end{align*}
where $d(u,v) = \frac{1}{2}\sum_i|x'_u(i) - x'_v(i)|$.
\end{lemma}
\begin{proof}[Proof of Lemma \ref{lemma:rounding-guar} (rounding guarantees)]
The output $x^{\alpha}$ is clearly an $i$-expansion of $x$ for the $i$ chosen in line 3. 

For the guarantees, first, fix $u\in V$ and a label $i\ne x(u)$. 
We output $x^{\alpha}(u) = i$ precisely when $i$ is chosen in line 3, and $0 < r < x'_u(i)$, which occurs with probability $\frac{1}{k}\frac{x'_u(i)}{1/k} = x'_u(i)$ (we used here that $x'_u(i) \le \epsilon < 1/k$ for all $i\ne x(u)$). Now we output $x^{\alpha}(u)=x(u)$ with probability $1-\sum_{j\ne x(u)}\P[x^{\alpha}(u) = j] = 1-\sum_{j\ne x(u)}x'_u(j) = x'_u(x(u))$, since $\sum_{i}x'_u(i) = 1$. This proves the first guarantee.

For the second, consider an edge $(u,v)$ not cut by $x$, so $x(u) = x(v)$. Then $(u,v)$ is cut by $x^{\alpha}$ when some $i \ne x(u)$ is chosen and $r$ falls between $x'_u(i)$ and $x'_v(i)$. This occurs with probability \[\frac{1}{k}\sum_{i\ne x(u)}\frac{\max(x'_u(i), x'_v(i))-\min(x'_u(i), x'_v(i))}{1/k} = \sum_{i\ne x(u)} |x'_u(i) - x'_v(i)| \le 2d(u,v).\]

Finally, consider an edge $(u,v)$ cut by $x$, so that $x(u) \ne x(v)$. Here $x^{\alpha}(u) = x^{\alpha}(v)$ if some $i,r$ are chosen with $r < \min(x'_u(i), x'_v(i))$. We have $r < \min(x'_u(i), x'_v(i))$ with probability $\frac{\min(x'_u(i), x'_v(i))}{1/k}$. Note that this is still valid if $i=x(u)$ or $i=x(v)$, since only one of those equalities can hold.
So we get \[\P[x^{\alpha}(u) = x^{\alpha}(v)] = \frac{1}{k}
\sum_{i} \frac{\min(x'_u(i),x'_v(i))}{1/k}  = \frac{1}{2}\left(\sum_i x'_u(i) + x'_v(i) - |x'_u(i) - x'_v(i)|\right) = 1 - d(u,v),
\]
where we used again that $\sum_i x'_u(i) = 1$.
\end{proof}  
Algorithm \ref{alg:rounding_algorithm} is very similar to the rounding algorithm from \citet{LanSonVij18}, essentially just using different constants to give a simplified analysis.
The algorithm used in \citet{LanSonVij18} was itself a simple modification of the $\epsilon$-close rounding from \citet{AngMakMak17}.

With these guarantees in hand, we can now prove \eqref{eqn:rounding-ineq-apdx}. Let $x^{\alpha} = R(x,y')$. Let $E^x = \{(u,v) \in E : x(u) \ne x(v)\}$ be the set of edges cut by $x$. Recall that $\theta_{uv}(i,j) = w_{uv}\mathbb{I}[i\ne j]$. Then we have:
\begin{align*}
    \E[\langle\theta, x - x^{\alpha}\rangle] = \sum_u\theta_u(x(u))\P[x^{\alpha}(u) \ne x(u)] - \sum_u\sum_{i\ne x(u)}\theta_u(i)\P[x^{\alpha}(u) = i] &+ \sum_{uv \in E^x}w_{uv}\P[x^{\alpha}(u) = x^{\alpha}(v)]\\
    &- \sum_{uv \in E\setminus E^x}w_{uv}\P[x^{\alpha}(u) \ne x^{\alpha}(v)].
\end{align*}
Applying Lemma \ref{lemma:rounding-guar}, we obtain:
\begin{align}
\E[\langle\theta, x - x^{\alpha}\rangle] &\ge \sum_u \theta_u(x(u))(1-x'_u(x(u))) - \sum_{u}\sum_{i\ne x(u)}\theta_u(i)x'_u(i) + \sum_{uv\in E^x}w_{uv}(1-d(u,v)) - \sum_{uv \in E\setminus E^x} 2w_{uv}d(u,v)\nonumber\\
&= \sum_u\theta_u(x(u)) + \sum_{uv\in E^x}w^x_{uv} - \sum_{u}\sum_i \theta_u(i)x'_u(i) - \sum_{uv\in E}w^x_{uv}d(u,v)\label{eqn:perturbed-objective}.
\end{align}
Here we are using the formula for $w^x_{uv}$ given by \eqref{eqn:wx-defn-apdx}: $w^x_{uv} = w_{uv}$ if $(u,v)$ is in $E^x$, and $2w_{uv}$ otherwise.

Because $x$ is a vertex of $M(G)$, the node variables $x_u(i)$ are either 0 or 1. Then there is only one setting of $x_{uv}(i,j)$ that satisfies the marginalization constraints. So the edge cost paid by $x$ on each edge is proportional to $\frac{1}{2}\sum_{uv}|x_u(i) - x_v(i)|$, since this is 1 if $x$ labels $u$ and $v$ differently and 0 otherwise. Therefore, 
\[
\sum_{uv}\sum_{i,j}\theta^x_{uv}(i,j)x_{uv}(i,j) = \sum_{uv}\frac{w^x_{uv}}{2}\sum_i|x_u(i) - x_v(i)|.
\]
The following proposition says we can also rewrite the edge cost paid by the LP solution $y'$ in this way.
 \begin{proposition*}
 In uniform metric labeling, there is a closed form for the optimal edge cost that only involves the node variables. That is, fix arbitrary node variables $z_u(i)$ and $z_v(j)$. Then the value of 
\begin{alignat*}{2}
    \min_{z_{uv}}\qquad &\sum_{i,j}\mathbb{I}[i\ne j]z_{uv}(i,j)\\
    \text{subject to } &\sum_j z_{uv}(i,j) = z_u(i)\\
    & \sum_{i} z_{uv}(i,j) = z_v(j)\\
    & z_{uv}(i,j) \ge 0
\end{alignat*}
is equal to $\frac{1}{2}\sum_{i}|z_u(i) - z_v(i)|$ \citep{archer2004approximate, LanSonVij18}. This fact is used to prove that the local LP relaxation is equivalent to the \emph{metric} LP relaxation for uniform metric labeling \citep{archer2004approximate,LanSonVij18}.
\end{proposition*}

Because $y'$ is an optimal solution to \eqref{eqn:local-lp} for objective $\theta^x$, $y'$ pays the minimum edge cost consistent with its node variables, since otherwise it cannot be optimal. Then the above proposition implies that:
\[
\sum_{uv}\sum_{i,j}\theta^x_{uv}(i,j)y'_{uv}(i,j) = \sum_{uv}\frac{w^x_{uv}}{2}\sum_i|y'_u(i) - y'_v(i)|.
\]
Since $x'_{uv}(i,j) = \epsilon y'_{uv}(i,j) + (1-\epsilon)x_{uv}(i,j)$, and $d(u,v)$ is convex,
\[
d(u,v) = \frac{1}{2}\sum_i|x'_u(i) - x'_v(i)| \le \frac{\epsilon}{2}\sum_i|y'_u(i) - y'_v(i)| + \frac{1-\epsilon}{2}\sum_i|x_u(i) - x_v(i)|
\]

Using this, the definition of $x'$, and the closed forms for the edge cost of $y'$ and $x$, we can simplify \eqref{eqn:perturbed-objective} to:
\begin{alignat*}{2}
\E[\dot{\theta^x}{x-x^{\alpha}}] &\ge \dot{\theta^x}{x} &&- \left[(1-\epsilon)\sum_u\theta_u(x(u)) + \epsilon\sum_u\sum_i\theta_u(i)y'_u(i) + (1-\epsilon)\sum_{uv}\frac{w^x_{uv}}{2}\sum_i|x_u(i) - x_v(i)| \right.\\
&&&+\left. \epsilon\sum_{uv}\frac{w^x_{uv}}{2}\sum_i|y_u(i) - y_v(i)|\right]\\
&= \dot{\theta^x}{x}&& - \left[(1-\epsilon)\dot{\theta^x}{x} + \epsilon\dot{\theta^x}{y'}\right]\\
&= \epsilon\dot{\theta^x}{x &&- y'},
\end{alignat*}
which is what we wanted to show. This analysis implies that for any expansion minimum $x$, (i) $x$ is a MAP solution to the instance $\theta^x$ and (ii) the local LP relaxation \eqref{eqn:local-lp} is tight on the instance $\theta^x$. Point (ii) is crucial for the correctness of our algorithm in Section \ref{sec:verifying}. However, in the next section we give a simpler proof of (i) that does not use the local LP relaxation.
\end{proof}

\subsection{Combinatorial proof of Theorem \ref{thm:exactpert} part (i)}
Here, we give a simpler proof for the first claim of Theorem \ref{thm:exactpert}, that a solution $x$ returned by $\alpha$-expansion is the optimal labeling in the instance with objective $\theta^x$.
However, the extra guarantee of Theorem \ref{thm:exactpert}, that the local LP relaxation is tight on the instance with objective $\theta^x$, was crucial to the correctness of our algorithm in Section \ref{sec:verifying}.
\begin{theorem*}
Consider an input instance $\theta$ with Potts pairwise potentials and weights $w$, and let the labeling $x$ be a local minimum for $\theta$ with respect to expansion moves. Define perturbed weights ${w^x: E\to \mathbb{R}_+}$ as
\begin{equation}
w^x_{uv} = \begin{cases}
    w_{uv} & x(u) \ne x(v)\\
    2w_{uv} & x(u) = x(v),
\end{cases}
\end{equation}
and let 
\begin{equation}
    \theta^x_{uv}(i,j) = w^x_{uv}\mathbb{I}[i\ne j]
\end{equation}
be the pairwise Potts energies corresponding to the weights $w^x$.
Then $x$ is a global minimum in the instance with objective vector $\theta^x = (\theta_u : u;\ \theta_{uv}^x : uv)$. 
This is the Potts model instance with the same node costs $\theta_u(i)$ as the original instance, but new pairwise energies $\theta_{uv}^x(i,j)$ defined using the perturbed weights $w^x$.
\end{theorem*}
\begin{proof}
We'll show that if some assignment $y$ obtains $\dot{\theta^x}{y} < \dot{\theta^x}{x}$, there exists an expansion move
$x^{\alpha}$ of $x$ with $\dot{\theta}{x^{\alpha}} < \dot{\theta}{x}$. Consequently, when $x$ is optimal with
respect to expansion moves, it is also the global optimal assignment
in the instance with objective $\theta^x$. Assume such a $y$ exists and define $V^{\alpha} =
\{u \in V | y(u) = \alpha\}.$ This is the set of points labeled $\alpha$ by $y$. The sets $(V^1,\ldots, V^k)$ form a partition of $V$. 
For each $\alpha \in [k]$, define the
expansion $x^{\alpha}$ of $x$ towards $y$ as:
\begin{equation*}
  x^{\alpha}(u) = \begin{cases}
    \alpha & u \in V^{\alpha}\\
    x(u) & \mbox{otherwise.}
  \end{cases}
\end{equation*}

We will show:
\begin{equation}
\label{eqn:lemma}
    \sum_{\alpha} \left(\dot{\theta}{x} - \dot{\theta}{x^{\alpha}}\right) \ge \dot{\theta^x}{x} - \dot{\theta^x}{y}    
\end{equation}
This immediately gives the
result: if $\dot{\theta^x}{y} < \dot{\theta^x}{x}$, then at least one term in the sum on the
left-hand-side must be positive, and this corresponds to an expansion
$x^{\alpha}$ of $x$ with better objective in the original instance.

  Consider a single term $\dot{\theta}{x}
  - \dot{\theta}{x^\alpha}$ on the left-hand-side of \eqref{eqn:lemma}. The difference in node cost
  terms is precisely $\sum_{u\in V^{\alpha}}\theta_u(x(u)) -
  \theta_u(x^{\alpha}(u))$, since on all $v \in V\setminus
  V^{\alpha}$, $x^{\alpha}(v) = x(v)$. This is equal to $\sum_{u\in V^{\alpha}}\theta_u(x(u)) - \theta_u(y(u))$, so the sum over $\alpha$ gives the difference in node cost between $x$ and $y$:
\begin{equation}\label{eqn:nodes-equal}
  \sum_{\alpha}\sum_{u\in V^{\alpha}}\theta_u(x(u)) - \theta_u(x^{\alpha}(u)) = \sum_{u\in V}\theta_u(x(u)) - \theta_u(y(u)).
  \end{equation}
  For any assignment $z$, let
  $E_z \subset E$ be the set of edges $(u,v)$ separated by $z$. 
   Then we can write the difference in edge costs between $x$ and
  $x^\alpha$, with the original weights $w_{uv}$, as
  \begin{equation*}
    \sum_{uv\in E_x \setminus E_{x^\alpha}}w_{uv} - \sum_{uv \in E_{x^\alpha}\setminus E_x}w_{uv},
  \end{equation*}
  and the edge cost difference between $x$ and $y$ with weights $w^x$ as:
  \begin{equation*}
     \sum_{uv\in E_x \setminus E_y}w_{uv} - \sum_{uv \in E_y\setminus E_x}2 w_{uv},
  \end{equation*}
  where we used the definition of $w^x_{uv}$.
  Then what remains is to show:
  \[
   \sum_{\alpha}\left(\sum_{uv\in E_x \setminus E_{x^\alpha}}w_{uv} - \sum_{uv \in E_{x^\alpha}\setminus E_x}w_{uv}\right) \ge \sum_{uv\in E_x \setminus E_y}w_{uv} - \sum_{uv \in E_y\setminus E_x}2 w_{uv}.
  \]
  Define $B^{\alpha}$ to be the set of edges with exactly one endpoint
  in $V^{\alpha}$ i.e., $B^{\alpha} = \{(u,v) \in E : |\{u,v\} \cap
  V^{\alpha}| = 1\}$. For all $(u,v) \in B^{\alpha}$, $y(u) \ne y(v)$,
  and either $y(u) = \alpha$ or $y(v) = \alpha$.  
  
  Let $(u,v) \in E_x \setminus E_y$. Because $y(u) = y(v)$, the edge $(u,v)$ appears in \emph{exactly one} of the
  $E_x\setminus E_{x^{\alpha}}$. That is, $y(u) = y(v) = \alpha$, so $x^{\alpha}$ does not cut $(u,v)$, and $x^{\beta}$ cuts $(u,v)$ for all $\beta\ne\alpha$. This implies
  \begin{equation}
  \label{eqn:cut-by-x}
  \sum_\alpha\sum_{uv\in E_x\setminus E_{x^{\alpha}}}w_{uv} \ge \sum_{uv \in E_x\setminus E_y}w_{uv}
  \end{equation}
  If $x^{\alpha}$ separates an edge $(u,v)$ that
  is not separated by $x$, exactly one endpoint of $(u,v)$ is in
  $V^{\alpha}$, since otherwise both endpoints would have been assigned label
  $\alpha$. Thus $E_{x^\alpha}\setminus E_x \subset B_\alpha \setminus
  E_x$. This implies
\begin{equation}
\label{eqn:cut-by-x-alpha}
  \sum_\alpha\sum_{uv\in E_{x^{\alpha}}\setminus E_x}w_{uv} = \sum_{\alpha}\sum_{uv \in B^{\alpha}\setminus E_x}w_{uv} = 2\sum_{uv\in E_y\setminus E_x}w_{uv},    
\end{equation}
where the last equality is because each edge in $E_y$ appears in two $B^{\alpha}$.
Combining \eqref{eqn:cut-by-x} and \eqref{eqn:cut-by-x-alpha}, we obtain:
\begin{equation}\label{eqn:edge-ineq}
  \sum_{\alpha}\left(\sum_{uv\in E_x \setminus E_{x^\alpha}}w_{uv} - \sum_{uv \in E_{x^\alpha}\setminus E_x}w_{uv}\right) \ge \sum_{uv\in E_x \setminus E_y}w_{uv} - \sum_{uv \in E_y\setminus E_x}2 w_{uv},
\end{equation}
which is what we wanted. Combining \eqref{eqn:nodes-equal} and \eqref{eqn:edge-ineq}, we obtain \eqref{eqn:lemma}.
\end{proof}

\subsection{Proof of Lemma \ref{lem:theta-x-enough}}
\begin{proof}[Proof of lemma \ref{lem:theta-x-enough}]
Recall that $\mathcal{S}(\theta)$ is defined as the set of $x$ for which there exists $\theta' \in \mathcal{I}(\theta)$ such that $x$ is a MAP solution to the instance $\theta'$. We want to show that $\mathcal{S}(\theta)$ can also be written as:
\begin{align*}
    \mathcal{S}(\theta) = \{x : x &\text{ a MAP solution to the instance } \theta^x \text{ defined by \eqref{eqn:wx-defn} and \eqref{eqn:thetax-defn}}\}
\end{align*}
To do show, we simply show that if $x$ is a MAP solution for some $\theta'\in \mathcal{I}(\theta)$, then $x$ is also the MAP solution to the instance $\theta^x$. This is effectively because $\theta^x$ is the ``best possible'' perturbation for $x$ that is contained in $\mathcal{I}(\theta)$.
Fix $\theta' \in \mathcal{I}(\theta)$ for which $x$ is a MAP solution. Then for all labelings $y \ne x$,
$\dot{\theta'}{y} \ge \dot{\theta'}{x}$. In particular,
\[
\sum_{u}\theta'_u(y(u)) + \sum_{uv}\theta'_{uv}(y(u),y(v)) \ge \sum_{u}\theta'_u(x(u)) + \sum_{uv}\theta'_{uv}(x(u),x(v)).
\]
Because we assume throughout that $\theta_{uv}(i,j) = w_{uv}\mathbb{I}[i \ne j]$ (i.e., that the input instance is a Potts model), the definition of $\mathcal{I}(\theta)$ (equation \ref{eqn:idefn}) implies
that every instance in $\mathcal{I}(\theta)$ is a Potts model. So let $w'$ be the weights of the instance $\theta'$. Additionally, recall that the definition of $\mathcal{I}(\theta)$ implies that $\theta'_u(i) = \theta_u(i)$ for all $(u,i)$. Then the inequality above becomes:
\[
\sum_{u}\theta_u(y(u)) - \sum_u\theta_u(x(u)) + \sum_{\substack{uv: y(u) \ne y(v) \\ x(u) = x(v)}}w'_{uv} -  \sum_{\substack{uv: x(u) \ne x(v) \\ y(u) = y(v)}}w'_{uv} \ge 0
\]

The definition of $\mathcal{I}(\theta)$ requires that for all $(u,v)$, $w_{uv} \le w'_{uv} \le 2w_{uv}$. Together with the previous inequality, this implies 
\[
\sum_{u}\theta_u(y(u)) - \sum_u\theta_u(x(u)) + \sum_{\substack{uv: y(u) \ne y(v) \\ x(u) = x(v)}}2w_{uv} -  \sum_{\substack{uv: x(u) \ne x(v) \\ y(u) = y(v)}}w_{uv} \ge 0.
\]
By definition of the perturbed weights $w^x_{uv}$ \eqref{eqn:wx-defn-apdx}, we have
\[
\sum_{u}\theta_u(y(u)) - \sum_u\theta_u(x(u)) + \sum_{\substack{uv: y(u) \ne y(v) \\ x(u) = x(v)}}w^x_{uv} -  \sum_{\substack{uv: x(u) \ne x(v) \\ y(u) = y(v)}}w^x_{uv} \ge 0,
\]
which is equivalent to:
\[
\dot{\theta^x}{y} \ge \dot{\theta^x}{x}.
\]
Because $y$ was arbitrary, this implies $x$ is a MAP solution to the instance $\theta^x$.
\end{proof}

\section{Comparing \eqref{eqn:maximization} and \eqref{eqn:naive-alg}}
\label{apdx:naive-bound}
In this section, we expound on the relationship between \eqref{eqn:maximization} and \eqref{eqn:naive-alg}, the bound obtained directly from $\alpha$-expansion's objective approximation guarantee.
In particular, we show that any $x$ that is feasible for \eqref{eqn:maximization} is also feasible for \eqref{eqn:naive-alg}. While we solve the relaxation \eqref{eqn:relaxed-alg} of \eqref{eqn:maximization} in practice, this gives some intuition for why \eqref{eqn:relaxed-alg} gives much tighter bounds than \eqref{eqn:naive-alg}.

We have two ways of characterizing the set of labelings $x$ that are local optima w.r.t. expansion moves. The first, guaranteed by \citet{BoyVekZab01}, is that all such $x$ satisfy
\begin{equation}
\label{eqn:apdx-boykov-guar}
\dot{\theta}{x} \le \dot{\theta}{x^*} + \sum_{uv\in E}w_{uv}\mathbb{I}[x^*(u)\ne x^*(v)],
\end{equation}
where $x^*$ is a MAP solution. That is, the ``extra'' objective paid by $x$ is at most the edge cost paid by a MAP solution.
The second, guaranteed by Theorem \ref{thm:localglobal}, is that $x$ is the MAP solution in the instance with objective $\theta^x$ (i.e., $x \in \mathcal{S}(\theta)$).
We now show that any labeling $x$ that is a MAP solution in the instance with objective $\theta^x$ also satisfies \eqref{eqn:apdx-boykov-guar}, but the converse is not true. This implies that the feasible region of \eqref{eqn:maximization} is strictly smaller than that of \eqref{eqn:naive-alg}.

\begin{proposition*}
Let $x$ be a labeling that is optimal in the instance with objective $\theta^x$, and let $x^*$ be a MAP solution to the original instance, with objective $\theta$. Then:
\[
\dot{\theta}{x} \le \dot{\theta}{x^*} + \sum_{uv\in E}w_{uv}\mathbb{I}[x^*(u)\ne x^*(v)],
\]
\end{proposition*}
\begin{proof}
Because $x$ is optimal for $\theta^x$, we have $\dot{\theta^x}{x} \le \dot{\theta^x}{x^*}$. Recall from the definitions of $w^x_{uv}$ and $\theta^x_{uv}(i,j)$ (\eqref{eqn:wx-defn-apdx} and \eqref{eqn:thetax-defn-apdx}) that $\dot{\theta^x}{x} = \dot{\theta}{x}$. We also have that
\begin{align*}
\dot{\theta^x}{x^*} = \sum_u\theta_u(x^*(u)) + \sum_{uv}w^x_{uv}\mathbb{I}[x^*(u)\ne x^*(v)] &\le \sum_u\theta_u(x^*(u)) + 2\sum_{uv}w_{uv}\mathbb{I}[x^*(u)\ne x^*(v)]\\
&= \dot{\theta}{x^*} + \sum_{uv}w_{uv}\mathbb{I}[x^*(u)\ne x^*(v)].
\end{align*}
Here we used that $w^x_{uv} \le 2w_{uv}$ for all $(u,v)\in E$.
Therefore, $\dot{\theta}{x} \le \dot{\theta}{x^*} + \sum_{uv\in E}w_{uv}\mathbb{I}[x^*(u)\ne x^*(v)]$.
\end{proof}
Conversely, not all $x$ satisfying \eqref{eqn:apdx-boykov-guar} are optimal in the instance with objective $\theta^x$. We now construct a simple example.

\paragraph{Example where \eqref{eqn:maximization} is much tighter than \eqref{eqn:naive-alg}.} Let $k=4$ and consider a graph $G = (V,E)$ with two nodes $s$ and $t$, and one edge $(s,t)$. Let $w_{st} = 1$.
For the node costs, Set $\theta_s(0) = 0$, $\theta_s(1) = \epsilon$, and $\theta_s(2) = \theta_s(3) = \infty$. Set $\theta_t(0) = \theta_t(1) = \infty$, $\theta_t(2) = \epsilon$, $\theta_t(3) = 0$. The MAP solution $x^*$ clearly labels $s$ with label 0 and $t$ with label 3, for an objective of 1. Now consider the solution $x$ that labels $s$ with label 1 and $t$ with label 2, for an objective of $1 + 2\epsilon$. For this $x$, because $x$ cuts the only edge, $\theta^x = \theta$ (see \eqref{eqn:wx-defn-apdx}). Therefore, $x$ is not optimal in the instance with objective $\theta^x$, so it is not feasible for \eqref{eqn:maximization}. However, 
\[
1 + 2\epsilon \le \dot{\theta}{x^*} + \sum_{uv\in E}w_{uv}\mathbb{I}[x^*(u)\ne x^*(v)] = 2
\]
for $\epsilon < 1/2$. Therefore, $x$ is feasible for \eqref{eqn:naive-alg}.
The Hamming distance between $x$ and $x^*$ is 1.0---$x$ agrees with $x^*$ on 0 nodes---so if we run \eqref{eqn:naive-alg} to bound the Hamming error on this instance, we obtain a bound of 1.0.
On the other hand, \eqref{eqn:maximization} returns a Hamming distance bound of 0 for this instance, correctly indicating that $\alpha$-expansion always recovers $x^*$ regardless of the initialization. This is because any optimal $x$ must cut $(s,t)$, since otherwise it has infinite objective, and any $x$ that cuts $(s,t)$ has $\theta^x = \theta$, and $x^*$ is the only optimal labeling for objective $\theta$. Hence $x^*$ is the only feasible point of \eqref{eqn:maximization} for this instance.

\section{Model details}
\label{apdx:model-details}
In this section, we give more details on the models used for our experiments in Section \ref{sec:exps}. 
These models are similar to the ones studied in \citet{LanSonVij19}.
There are two types of models: object segmentation and stereo vision.

\subsection{Object segmentation}
We use the object segmentation models from \citet{shotton2006textonboost}, which were also studied by \citet{alahari2010dynamic} in the context of graph cut methods. These models are available as part of the OpenGM 2 benchmark \citep{kappes2015comparative}\footnote{All OpenGM 2 benchmark models are accessible at \url{http://hciweb2.iwr.uni-heidelberg.de/opengm/index.php?l0=benchmark}}. 
In these models, $G$ is a grid with one vertex per pixel and has edges connecting adjacent pixels.
The node costs $\theta_u(i)$ are set based on a learned function of shape, color, and location features. Similarly, the edge weights are set using \emph{contrast-sensitive} features:
\[
w_{uv} = \eta_1\exp\left(-\frac{||I(u) - I(v)||_2^2}{2\sum_{p,q}||I(p) - I(q)||_2^2}\right) + \eta_2,
\]
where $\eta = (\eta_1, \eta_2)$, $\eta \ge 0$ are learned parameters, and $I(u)$ is the vector of RGB values for pixel $u$ in the image.
\citet{shotton2006textonboost} learn the parameters for the node and edge potentials using a shared boosting method.
Each object segmentation instance has 68,160 nodes (the images are {\tt 213 x 320}) and either $k=5$ or $k=8$ labels.
As noted in \citet{kappes2015comparative}, the MRFs used in practice increasingly use potential functions that are learned from data, rather than set by hand.
In our experimental results, we found that both the objective gap and Hamming distance bounds were very good for these instances (in comparison to the stereo examples, which have ``hand-set'' potentials).
Do the learning dynamics automatically encourage solutions to perturbed instances to be close to solutions of the original instance?
Understanding the relationship between learning and this ``stability'' property is an interesting direction for future work.

\subsection{Stereo Vision}
In these models, the weights $w_{uv}$ and costs $\theta_u(i)$ are set ``by hand'' according to the model from \citet{tappen2003comparison}. Given two images taken from slightly offset locations, the goal is to estimate the depth of every pixel in one of the images. This can be done by estimating, for each pixel, the disparity between the two images, since the depth is inversely proportional to the disparity. In the \citet{tappen2003comparison} model, the node costs are set using the sampling-invariant technique from \citet{birchfield1998pixel}. These costs are similar to
\[
\theta_u(i) = (I_L(u) - I_R(u-i))^2,
\]
where $I_L$ and $I_R$ are the pixel intensities in the left and right images. If node $u$ corresponds to pixel location $(h,w)$, we use $u-i$ to represent the pixel in location $(h,w-i)$. So this cost function measures how likely it is that the pixel at location $u$ in the left image corresponds to the pixel at location $u-i$ in the right image. The Birchfield-Tomasi matching costs are set using a correction to this expression that accounts for image sampling.
In the Tappen and Freeman model, the weights $w_{uv}$ are set as:
\begin{equation*}
  w_{uv} = \begin{cases}
    P \times s & |I(u) - I(v)| < T\\
    s & \text{otherwise},
  \end{cases}
\end{equation*}
where $P,T,$ and $s$ are the parameters of the model, and $I(u)$ is the intensity of pixel $u$ in one of the input images to the stereo problem (in our experiments, we use $I_L$, the left image). These edge weights charge more for separating pixels with similar intensities, since nearby pixels with similar intensities are likely to correspond to the same object, and therefore be at the same depth.
In our experiments, we follow \citet{tappen2003comparison} and set
$s=50$, $P = 2$, $T = 4$. In our experiments, we used images from the Middlebury stereo dataset \citep[see, e.g.,][]{scharstein2014high}. We used a downscaled version of the {\tt tsukuba} image that was {\tt 120 x 150}, and had $k=7$. Our {\tt venus} model used the full-size image, which is {\tt 383 x 434}, and has $k=5$. For {\tt plastic}, we again used a downscaled, {\tt 111 x 127} image with $k=5$. The large size of the {\tt venus} image, in particular, shows that our verification algorithm is tractable to run even on fairly large problems.

\end{document}